\pdfoutput=1
\documentclass{article}
\usepackage{iclr2027_conference,times}

\usepackage{amsmath,amssymb,amsthm}
\usepackage{mathtools}
\usepackage{booktabs}
\usepackage{graphicx}
\usepackage{multirow}
\usepackage{enumitem}
\usepackage{xcolor}
\usepackage{hyperref}
\usepackage{url}

\iclrfinalcopy

\newtheorem{theorem}{Theorem}

\newtheorem{corollary}[theorem]{Corollary}

\theoremstyle{definition}

\theoremstyle{plain}  % remarks set like the propositions: bold heading, same body

\newcounter{restatesaved}
\NewDocumentEnvironment{restatable}{o m m +b}{%
  \IfNoValueTF{#1}{\begin{#2}#4\end{#2}}{\begin{#2}[#1]#4\end{#2}}%
  \expandafter\xdef\csname #3number\endcsname{\arabic{theorem}}%
  \expandafter\gdef\csname #3\endcsname{\restateresult{#2}{#3}{#1}{#4}}}{}
\NewDocumentCommand{\restateresult}{m m m +m}{%
  \begingroup
  \setcounter{restatesaved}{\value{theorem}}%
  \setcounter{theorem}{\numexpr\csname #2number\endcsname-1\relax}%
  \renewcommand{\label}[1]{}%
  \renewcommand{\theHtheorem}{restated.\arabic{theorem}}%
  \IfNoValueTF{#3}{\begin{#1}#4\end{#1}}{\begin{#1}[#3]#4\end{#1}}%
  \setcounter{theorem}{\value{restatesaved}}%
  \endgroup}

\newcommand{\E}{\mathbb{E}}
\newcommand{\R}{\mathbb{R}}
\newcommand{\piref}{\pi_{\mathrm{ref}}}
\newcommand{\pitheta}{\pi_\theta}
\newcommand{\method}{BA-DPO}
\title{BA-DPO: Bias-Adjusted Direct Preference\\ Optimization for Language Model Alignment}

\author{Antonio Ferrara\thanks{Equal contribution.} \\
Intesa Sanpaolo AI Research
\And
Alberto Rumi\footnotemark[1] \\
Intesa Sanpaolo AI Research
\And
Francesco Bonchi \\
Intesa Sanpaolo AI Research
}

\begin{document}
\addtocontents{toc}{\protect\setcounter{tocdepth}{-1}}

\maketitle
\lhead{Preprint.}

\begin{abstract}

Preference-based alignment methods such as Direct Preference Optimization (DPO) use pairwise preferences labeled by human annotators to fine-tune language models. However, annotators carry systematic biases toward some attributes: a name that signals a gender or an ethnicity, a persona, a language variety, a formatting convention, or length. If not properly addressed, these systematic biases can be absorbed and amplified during alignment. Existing methods address length bias or annotator disagreement, but fail to eliminate biases toward arbitrary attributes. To address this limitation, we propose \emph{Bias-Adjusted DPO} (BA-DPO), a generalization of DPO that adds one bias parameter per annotator toward responses carrying a declared attribute. 
We prove that the objective is convex in the bias parameters and that the votes identify each annotator's bias up to a shared constant. The remaining constant is what fixes the aligned model's attribute rate: by default the rate of the reference model, or a target rate, which we use to bring a biased policy to statistical parity.

On a corpus with planted biases, DPO drives the attribute from a balanced start to probability $0.96$ and BA-DPO removes $81$ to $95\%$ of that shift; on MultiPref with real annotators it removes about half of DPO's lengthening. Both hold at 0.5B with full fine-tuning and at 8B with LoRA, at no higher KL than DPO and no loss in judged quality.

\end{abstract}

\section{Introduction}
\label{sec:intro}

Preference-based alignment has become the dominant paradigm for shaping language model behavior from pairwise preference labels produced by human annotators. Standard techniques, whether they rely on reinforcement learning with a reward model \citep{christiano2017deep, ouyang2022training, bai2022hh} or optimize a cross-entropy loss over preference pairs directly, as in Direct Preference Optimization (DPO) \citep{dpo}, inherit the foundational assumption of the \emph{Bradley-Terry model} (BT) \citep{BT}: that every human-annotated pairwise comparison is a noisy but unbiased measurement of the quality gap between two responses. This assumption rarely holds. Human evaluators carry systematic preferences for attributes unrelated to quality, such as a name that signals a gender or an ethnicity \citep{bertrand2004emily}, a persona, a language variety, a formatting convention, or response length \citep{park2024disentangling, lu2024sampo, liu2024ld}. Because standard preference optimization treats every label as evidence of quality, systematic biases shared across annotators do not average out; instead, they are absorbed and amplified by the aligned policy as a preference for the attribute itself.

In practice, preference optimization cannot discern why a label was assigned. For instance, if annotators consistently favor answers signed with names signaling a specific demographic group, the objective simply treats the signature as a sign of response quality. Under standard DPO the effect is large: a policy that initially signs across demographics with equal probability shifts to placing probability $0.96$ on the favored marker. The same happens with real annotators who favor longer answers or Markdown styling (Section~\ref{sec:experiments}).

In this paper, we address the following alignment problem: \textit{given preference judgments from annotators and a declared attribute of the responses, how can we train a policy that does not inherit the annotators' shared bias toward that attribute?}

We introduce \textbf{Bias-Adjusted DPO (BA-DPO)}, which generalizes DPO to possibly biased annotators, where each annotator carries a parameter for how much they favour the declared attribute. The parameter is added to the reward and does not enter the partition function, so the reparameterization that lets DPO train without a reward model goes through unchanged (Proposition~\ref{prop:compat}).\\
We then characterise what the labels determine. The votes fix every annotator's bias relative to the others. What they leave open is one direction: raising all the biases together while lowering how much the policy itself favours the attribute changes no judgment (Theorem~\ref{thm:identif}, Corollary~\ref{prop:degeneracy}). That direction is what sets the attribute rate of the trained policy. Training from the reference leaves the rate the reference had, and stating a target rate moves it instead.

BA-DPO fundamentally differs from existing alignment corrections in its underlying formulation. Single-attribute heuristics like R-DPO \citep{park2024disentangling} and SamPO \citep{lu2024sampo} modify or regularize length distributions; the attribute is fixed in the objective, and as our experiments show, correcting it leaves another one in place or substitutes one spurious shortcut for another \citep{lamparth2026substitution}. Heterogeneity techniques like Group-DRO \citep{sagawa2020groupdro} and EM-DPO \citep{chidambaram2024emdpo} represent or reweight evaluator disagreement, which addresses how annotators differ rather than what they hold in common. Methods like ADPO \citep{adpo} remove unindexed response-level offsets, while BiasDPO \citep{biasdpo} addresses the converse problem of cleansing biased model outputs using clean preference data. By contrast, BA-DPO models label bias directly, as additive per-annotator scalars on declared attributes, which separates the bias the annotators share from the quality signal and removes it rather than accommodating it.
A more detailed discussion of prior related literature is provided in Appendix \ref{sec:related}.

The primary contributions of this paper are summarized as follows:
\begin{enumerate}[nosep,leftmargin=*]

    \item \textbf{Problem formulation and method:} We formalize the challenge of removing shared, attribute-directed annotator bias from preference optimization and propose BA-DPO, an efficient objective requiring one parameter per attribute when annotators are pooled and one per annotator and attribute otherwise, with no reward model and no dedicated network (Section~\ref{sec:method}).
        
    \item \textbf{Theoretical characterization:} We prove that the objective is convex in the bias parameters and that the votes fix every annotator's bias relative to the others, leaving one direction that sets the trained policy's attribute rate (Section~\ref{sec:theory}).
        
    \item \textbf{Empirical validation:} On a corpus with planted biases and on MultiPref \citep{multipref} with real annotators, at 0.5B with full fine-tuning and at 8B with LoRA, we show that BA-DPO removes $81$ to $95\%$ of DPO's attribute shift on the planted corpus and about half of its lengthening on MultiPref, at no higher KL than DPO and no loss in judged quality (Section~\ref{sec:experiments}).
        
\end{enumerate}

\section{Background}
\label{sec:background}

Given a prompt $x$ and two responses $y_1, y_2$, the Bradley-Terry (BT) model expresses the probability that $y_1$ is preferred over $y_2$ (denoted $y_1 \succ y_2$) in terms of the difference in intrinsic quality between $y_1$ and $y_2$:
\begin{equation}
p(y_1 \succ y_2 \mid x) = \sigma\left(r^\star(x, y_1) - r^\star(x, y_2)\right)
\end{equation}
where $r^\star$ is the latent reward function and $\sigma$ is the logistic sigmoid.

Reinforcement Learning with Human Feedback (RLHF) fits a reward model to pairwise comparisons and subsequently maximizes expected reward under a KL-divergence penalty of strength $\beta > 0$ relative to a reference policy $\pi_{\text{ref}}$:
\begin{equation}\label{eq:rl_objective}
\max_{\pi_\theta} \mathbb{E}_{x \sim \mathcal{D}, y \sim \pi_\theta(\cdot \mid x)} \left[ r(x, y) \right] - \beta \mathbb{D}_{\text{KL}}\left[\pi_\theta(y \mid x) \,\Vert\, \pi_{\text{ref}}(y \mid x)\right].
\end{equation}
\citet{dpo} observe that this constrained optimization problem has an exact closed-form solution 
\begin{equation}
 \pi_\theta(y \mid x) = \frac{1}{Z(x)} \pi_{\text{ref}}(y \mid x) \exp\left(\frac{1}{\beta} r(x, y)\right)
\end{equation} where $Z(x) = \sum_y \pi_{\text{ref}}(y \mid x) \exp\left(\frac{1}{\beta} r(x, y)\right)$ is the partition function. Rearranging terms yields the implicit reward:
\begin{equation}\label{eq:reward_reparam}
r(x, y) = \beta \log \frac{\pi_\theta(y \mid x)}{\pi_{\text{ref}}(y \mid x)} + \beta \log Z(x).
\end{equation}
When substituted back into the BT preference model, the partition function $Z(x)$ cancels out in the reward difference. Maximum likelihood estimation over preference pairs thus simplifies to a direct loss over the policy parameters, bypassing explicit reward modeling and reinforcement learning:
\begin{equation}\label{eq:dpo_loss}
\mathcal{L}_{\text{DPO}}(\pi_\theta; \pi_{\text{ref}}) = -\mathbb{E}_{(x, y_w, y_l) \sim \mathcal{D}} \left[ \log \sigma \left( \beta \log \frac{\pi_\theta(y_w \mid x)}{\pi_{\text{ref}}(y_w \mid x)} - \beta \log \frac{\pi_\theta(y_l \mid x)}{\pi_{\text{ref}}(y_l \mid x)} \right) \right].
\end{equation}

The BARP model \citep{barp} extends Bradley-Terry to account for evaluator bias in pairwise ranking. Item $i$ has a true quality score $s_i$ and a group attribute indicator $\delta_{g_i} \in \{0, 1\}$. Evaluator $k$ carries a scalar bias $\theta_k$ and perceives the score as $s_i + \theta_k \delta_{g_i}$, so that:
\begin{equation}
p_k(i \succ j) = \sigma\left(s_i + \theta_k \delta_{g_i} - s_j - \theta_k \delta_{g_j}\right).
\end{equation}
Scores and bias parameters are estimated jointly via maximum likelihood. The bias terms vanish for same-group comparisons ($\delta_{g_i} = \delta_{g_j}$), which isolate true item quality, whereas cross-group comparisons identify evaluator bias. We adopt this formulation as our foundation for developing BA-DPO.

\section{Bias-Adjusted DPO}
\label{sec:method}
Given a prompt $x$, an item (a response to the prompt by the language model) $y$ has a 
latent reward $r^\star(x,y)$ and a group indicator $\delta_{g(x,y)} \in \{0,1\}$ recording whether it exhibits
the attribute whose influence we wish to remove. Annotator $k \in \{1,\dots,m\}$ carries
a bias parameter $\theta_k \in \R$. Under the BARP model \citep{barp}, annotator $k$
prefers $y_1$ to $y_2$ with probability
\begin{equation}\label{eq:biased_bt}
    p_k(y_1 \succ y_2 \mid x) = \sigma\!\bigl(
        r^\star(x,y_1) + \theta_k\,\delta_{g(x,y_1)} - r^\star(x,y_2) - \theta_k\,\delta_{g(x,y_2)}\bigr),
\end{equation}
so that the annotator acts on a perceived reward
$\tilde{r}_k(x,y) = r^\star(x,y) + \theta_k\,\delta_{g(x,y)}$ rather than on $r^\star$
itself. Substituting the DPO reparameterization \eqref{eq:reward_reparam} into
\eqref{eq:biased_bt}, the two copies of $\beta\log Z(x)$ cancel exactly as they do in
standard DPO, because the bias terms are additive to the reward and independent of the
partition function. Given a dataset
$\mathcal{D} = \{(k^{(i)}, x^{(i)}, y_w^{(i)}, y_l^{(i)})\}_{i=1}^N$ in which each
comparison is attributed to its annotator; by applying maximum likelihood over $\mathcal{D}$ we obtain the \method{} loss:
\begin{equation}\label{eq:barp_dpo_loss}
    \mathcal{L}_{\text{BA-DPO}}(\pitheta, \{\theta_k\}; \piref)
    = -\E_{(k,x,y_w,y_l)\sim\mathcal{D}}\Bigl[\log\sigma\bigl(u + b_k\bigr)\Bigr],
\end{equation}
where we write
\begin{equation}\label{eq:shorthand}
    u \coloneqq \beta\log\frac{\pitheta(y_w \mid x)}{\piref(y_w \mid x)}
             - \beta\log\frac{\pitheta(y_l \mid x)}{\piref(y_l \mid x)},
    \qquad
    b_k \coloneqq \theta_k\bigl(\delta_{g(x,y_w)} - \delta_{g(x,y_l)}\bigr).
\end{equation}
We write $\Delta\delta \coloneqq \delta_{g(x,y_w)} - \delta_{g(x,y_l)} \in \{-1,0,1\}$
for the group difference of a comparison, so that $b_k = \theta_k\,\Delta\delta$. The
loss is optimised jointly over the policy parameters and the $m$ scalar bias
parameters. 

\begin{restatable}[\method{} reduces to DPO for unbiased annotators]{remark}{remdpospecial}\label{rem:dpo_special}
Setting $\theta_k = 0$ for every annotator in \eqref{eq:barp_dpo_loss} recovers the DPO loss
\eqref{eq:dpo_loss} exactly. DPO is therefore the special case of \method{} in which the
annotators are assumed unbiased, which is the assumption of the Bradley-Terry model.
\end{restatable}

\paragraph{Gradient analysis.}
The gradient of \eqref{eq:barp_dpo_loss} with respect to the policy parameters is the
DPO gradient with the importance weight $\sigma(-u)$ replaced by $\sigma(-u - b_k)$,
\begin{equation}\label{eq:grad_policy}
    \nabla_\phi \mathcal{L}_{\text{BA-DPO}} = -\beta\,\E_{(k,x,y_w,y_l)\sim\mathcal{D}}
    \Bigl[\sigma(-u - b_k)\cdot\bigl(\nabla_\phi\log\pi(y_w \mid x) - \nabla_\phi\log\pi(y_l \mid x)\bigr)\Bigr].
\end{equation}
When the annotator's bias favours the winner ($b_k > 0$) the weight shrinks and the
comparison teaches less, since part of its outcome is explained by bias; when the winner
was preferred against the bias ($b_k < 0$) the weight grows; when the two responses
share a group, $b_k = 0$ and the update is DPO's. The gradient with respect to a bias
parameter is the scalar
\begin{equation}\label{eq:grad_bias}
    \frac{\partial \mathcal{L}_{\text{BA-DPO}}}{\partial \theta_k}
    = -\frac{n_k}{N}\,\E_{(x,y_w,y_l)\sim\mathcal{D}_k}\Bigl[\sigma(-u - b_k)\cdot\bigl(\delta_{g(x,y_w)} - \delta_{g(x,y_l)}\bigr)\Bigr],
\end{equation}
where $\mathcal{D}_k$ are the comparisons provided by annotator $k$ and $n_k/N$ their
share of the data. The bias channel adds $m$ scalar parameters, a negligible cost next to the policy; the factor $n_k/N$ is decisive for the optimisation dynamics (see Proposition~\ref{prop:rate} in Appendix \ref{app:proof_rate}).

\label{sec:variants}

\paragraph{Per-annotator, pooled, and shared-mean.}
The model above requires annotator identities. When they are unavailable the bias
channel collapses to a single population-level scalar $\theta \in \R$ shared by every
comparison (the \emph{pooled} variant). When they are available there are two ways to use them: the \emph{naive per-annotator} variant with free parameters
$\{\theta_k\}$, and the \emph{shared-mean} variant
$\theta_k = \bar\theta + \varepsilon_k$ in which the mean is a single parameter updated
on every cross-group comparison and the deviations $\varepsilon_k$ carry the individual
structure. When the annotators belong to known classes, a class offset can be inserted between the two, $\theta_k = \bar\theta + \delta_c + \varepsilon_k$. All variants span the same
model class where they overlap; they differ in optimisation dynamics, and
Section~\ref{sec:mechanism} and Appendix~\ref{app:ablations} show that this difference,
not the identity information itself, decides how much bias is removed. % As controls we also train the per-annotator variant with deliberately shuffled annotator ids, and warm-start $\bar\theta$ (or the pooled $\theta$) at the corpus-level offline estimate, the log-odds that the $\delta_g = 1$ side wins a cross-group comparison. // This is about experiments 

\paragraph{Multiple attributes.}
For $q$ binary or continuous attributes with group vectors $\vec{g}_y \in \R^q$, the
bias term generalises to $b_k = \vec{\theta}_k^\top(\vec{g}_{y_w} - \vec{g}_{y_l})$ with
$\vec{\theta}_k \in \R^q$ per annotator. The Signed-UltraFeedback corpus of Section~\ref{sec:experiments}
uses $q = 2$ binary attributes, and Appendix~\ref{app:abl_variants} shows that declaring only one of the two leaves the other in the policy and removes only part of the declared one.

\section{Theoretical Characterization}
\label{sec:theory}
This section presents the theoretical characterization of \method{} (proofs in Appendix~\ref{app:proofs}).

The first result states precisely the substitution already presented in Section~\ref{sec:method}.
\begin{restatable}[Compatibility with the DPO reparameterization]{proposition}{propcompat}\label{prop:compat}
Let $r(x,y)$ be any reward function for which
$Z(x) = \sum_y \piref(y \mid x)\exp(r(x,y)/\beta)$ is finite for every prompt, and let
$\{\theta_k\}_{k=1}^m$ be arbitrary evaluator bias parameters. Then the bias-aware
preference model \eqref{eq:biased_bt} can be written equivalently as
\[
    p_k(y_1 \succ y_2 \mid x) = \sigma\!\left(\beta\log\frac{\pi(y_1 \mid x)}{\piref(y_1 \mid x)}
    - \beta\log\frac{\pi(y_2 \mid x)}{\piref(y_2 \mid x)}
    + \theta_k\bigl(\delta_{g(x,y_1)} - \delta_{g(x,y_2)}\bigr)\right)
\]
with $\pi(y \mid x) = \piref(y \mid x)\exp(r(x,y)/\beta)/Z(x)$.
\end{restatable}

Proposition~\ref{prop:compat} is what keeps the construction at one policy: the bias term is added to the reward and does not involve the partition function $Z(x)$, so it passes through the reparameterization unchanged and stays outside the policy. Methods for annotator disagreement differ in how much freedom they give each annotator's reward. DPO gives none: one reward is fit for everyone, so a bias the annotators share is represented nowhere except in the policy. Mixture methods give each of $K$ annotator types its own reward, and since each reward has its own optimal policy, they train $K$ policies \citep{chidambaram2024emdpo}. \method{} lets annotators differ only in the coefficient $\theta_k$ on a declared direction $\delta_g$, a fixed function of $(x,y)$ with no parameters of its own. A single policy then carries the reward the annotators share, the annotator-specific part enters the logit as the scalar $\theta_k$, and modelling annotator bias costs $m$ scalars, not $K$ policies.

We now characterise what the observed votes determine about the bias parameters. Part
of what they leave undetermined is inherited from DPO, which cannot see a reward shift
that is constant within a compared pair, because rewards enter the likelihood only
through the difference between the two responses of a pair. The rest is what the bias
channel adds, and the theorem holds for an otherwise arbitrary reward function.

Consider the graph whose vertices are the annotators, with an edge between two annotators
whenever they judged at least one cross-group comparison in common. We say that the
annotator pool is \emph{connected} if this graph is connected.

\begin{restatable}[Identifiability of the bias parameters]{theorem}{thmidentif}\label{thm:identif}
The loss $\mathcal{L}$ of \eqref{eq:barp_dpo_loss}, viewed as a function of the reward
values $r(x,y)$ and of $\theta$, is jointly convex in $(r, \theta)$,
and all its minimisers assign the same probability to every judgment. If the annotator
pool is connected and $(\hat r, \hat\theta)$ is a minimiser, the minimisers are exactly
the pairs
\[
    \bigl(\hat r + c\,\delta_g + h,\ \hat\theta - c\mathbf{1}\bigr), \qquad c \in \R,
\]
with $h$ any function constant on the two responses of every compared pair. Every
difference $\theta_k - \theta_{k'}$ therefore takes the same value at every minimiser,
while their mean $\bar\theta \coloneqq \frac{1}{m}\sum_k \theta_k$ is not determined by the judgments.
\end{restatable}

Connectedness requires that every annotator judged at least one
cross-group comparison, since the votes of an annotator who never saw the attribute on
both sides say nothing about their $\theta_k$. It holds whenever annotators are assigned
to comparisons at random, as in every corpus here; if it fails, the theorem applies
within each connected group of annotators, each with its own offset. For a fixed policy,
a minimiser in $\theta$ exists unless some annotator chose the attribute side in every
one of their cross-group comparisons, in which case their $\theta_k$ diverges.

Theorem~\ref{thm:identif} is stated in terms of the individual $\theta_k$; it can be restated by decomposing the annotators' biases into a shared mean and deviations.

\begin{restatable}[Reward-bias degeneracy]{corollary}{cordegeneracy}\label{prop:degeneracy}
Write $\theta_k = \bar\theta + \varepsilon_k$ with $\sum_k \varepsilon_k = 0$. Shifting
the mean by $-c$ and adding $c\,\delta_g$ to the reward leaves every judgment
probability unchanged, so $\bar\theta$ is not determined by the votes. If the annotator
pool is connected, the deviations $\varepsilon_k$ take the same value at every
minimiser.
\end{restatable}

The degeneracy is intrinsic rather than an artefact of the parameterisation: the
likelihood cannot distinguish between the attribute genuinely carrying reward and every
annotator being uniformly biased toward it. The free constant $c$ is therefore a choice,
and what it decides is how much the trained policy favours the attribute: by
Corollary~\ref{prop:degeneracy}, moving $c$ from $\bar\theta$ into the reward multiplies
the policy's odds for a response over its counterfactual without the attribute by
$e^{c/\beta}$ and changes no judgment probability. Two anchors are natural. Anchoring to
the reference takes $c$ such that the policy gives a response and its counterfactual the
same odds as the reference does, so the shared level of the labels sits in $\bar\theta$
and the policy keeps the reference's attribute rate; every experiment in the paper starts
from $\piref$ with $\theta = 0$ and ends there (Appendix~\ref{app:anchor_reference}).
Anchoring to a target takes $c$ such that the outputs meet a stated criterion, for
instance statistical parity, an attribute rate of one half: the required $c$ is found by
a one-dimensional search on the rate and applied as a fixed offset
(Appendix~\ref{app:anchor_target}; Appendix~\ref{app:anchor_experiment} uses it to bring
a biased DPO policy to parity). Which anchor is right is a normative decision, like the
attribute partition itself; the labels do not make it.

\begin{restatable}[Representability of the bias components]{proposition}{propabsorb}\label{prop:absorb}
Decompose $\theta_k = \bar{\theta} + \varepsilon_k$ with $\sum_k \varepsilon_k = 0$. The
shared component $\bar{\theta}\,\delta_{g(x,y)}$ is expressible as a function of $(x,y)$
and can therefore be represented inside the policy's implicit reward, whereas the
deviation term $\varepsilon_k\,\delta_{g(x,y)}$ depends on $k$, which the policy never
observes, and is representable by no policy $\pi(y \mid x)$ unless every $\varepsilon_k = 0$,
that is, unless all annotators have the same bias.
\end{restatable}

\label{sec:mechanism}
Proposition~\ref{prop:absorb} says what the policy could express, not what training
makes it express. Had the policy taken up the shared component, the bias parameters would
be left with the deviations alone, and the naive per-annotator and shared-mean variants
would then behave alike, since they write the deviations the same way. In our experiments
the shared component goes to the bias parameter instead: a mean started at zero and a mean
started at the offline estimate converge to the same interior value
(Appendix~\ref{app:abl_multipref}). The two variants therefore differ only in how they
write the same $m$ parameters, as $m$ free $\theta_k$ or as a mean plus deviations, and
that decides how fast the mean moves. The shared parameter is updated by every
cross-group judgment, while each free $\theta_k$ is updated only by the judgments of its
own annotator. Under gradient descent the naive mean is therefore exactly $m$ times slower
(Proposition~\ref{prop:rate}).

Together, these results explain the pattern of the experiments. The deviations are the
only component the votes determine on their own (Theorem~\ref{thm:identif}) and the only
one the policy cannot express (Proposition~\ref{prop:absorb}). Debiasing the policy,
however, requires the mean, which is the component the naive per-annotator variant learns
most slowly and the shared-mean variant learns first.

\section{Experiments}
\label{sec:experiments}

In this section, we test whether DPO absorbs a bias shared by the annotators and whether
\method{} removes it, on two corpora: Signed-UltraFeedback (built from UltraFeedback
\citep{ultrafeedback}), where the annotators' biases are planted and the amount to remove is
therefore known, and MultiPref \citep{multipref}, with real annotators, under length and
under formatting. 
DPO takes up the bias in every setting and model tested, and \method{} removes most of it without moving further
from the reference than DPO does.

\subsection{Experimental settings}
\label{sec:setup}
\paragraph{Datasets.}
The method needs an attribute that can be read from each response and, for the
per-annotator variants, the identity of the annotator behind each judgment. We evaluate
four attributes on two corpora (Appendix~\ref{app:datasets} has the full details).

Length and formatting can be read from any preference corpus. We use MultiPref
\citep{multipref}, one of the few public corpora with annotator identities: after
removing ties, $30{,}847$ disaggregated judgments on $10{,}461$ comparisons, each judged
by four of $227$ annotators. A response is \emph{long} when it has at least $1.5$ times
the words of its partner and \emph{formatted} when it contains Markdown.

We know of no public preference corpus with identified annotators in which the responses
carry a gender- or race-coded marker. We therefore introduce \emph{Signed-UltraFeedback},
built from the prompts and response pairs of UltraFeedback \citep{ultrafeedback} by ending
each response with a signature whose first name comes from published audit lists
\citep{bertrand2004emily,caliskan2017semantics} and is both gender-coded (woman or man)
and race-coded (white or black), so that every response carries two binary attributes at
once. Sixty annotators in three classes of $20$ vote on
each pair, four per pair, through the biased Bradley-Terry likelihood
\eqref{eq:biased_bt}, each with a bias vector $\theta_k \sim \mathcal{N}(\mu_c, 0.8^2
I)$. The class means are planted so that both attributes have a population mean near
$1.0$ in log-odds, the classes agree on the gender-coded attribute, and on the race-coded
one a class of $20$ leans the other way. % (Appendix~\ref{app:datasets} gives the values).
The $\theta_k$ are used only to generate the labels and to score the outcome, never by a
training method. The reference policy signs with each name group equally often, so any
bias in a trained policy comes from the preference stage. We say ``gender-coded'' and
``race-coded'' because the corpus tracks how a text marker propagates through alignment,
not human prejudice.

\paragraph{Models and methods.}
We run every experiment on two models, \texttt{Qwen2.5-0.5B-Instruct} \citep{qwen25}
and \texttt{Llama-3.1-8B-Instruct} \citep{llama3}, which differ in scale and in family.
The 0.5B model is fully fine-tuned; the 8B model is trained with LoRA \citep{hu2022lora}
on all linear layers. Each model
and corpus has its own SFT reference, fine-tuned on the majority-chosen response of each
comparison and shared by every method and seed, so that implicit rewards are comparable
(Appendix~\ref{app:training}). A third model, \texttt{Mistral-7B-Instruct-v0.3}
\citep{mistral7b}, of a different family and with a different tokenizer, repeats the comparison
between DPO and \method{} as a robustness check (Appendix~\ref{app:abl_mistral}).

We compare the reference, DPO, and the two \method{} variants of
Section~\ref{sec:variants}: \emph{pooled}, one scalar per attribute shared by every
annotator, and $\bar\theta+\varepsilon_k$, a shared mean per attribute plus one
deviation per annotator. 
The remaining variants of the bias model %, the free $\theta_k$ of the original formulation, shuffled annotator identities and warm starts, 
are controls for the mechanism, reported at 0.5B in
Appendix~\ref{app:ablations}.
r3

\paragraph{Baselines.}
To the best of our knowledge, no method for preference optimisation takes the attribute
as an input that the user supplies and can change between runs. We therefore compare with
the two families that come closest: the corrections built for length bias, which fix that
attribute in the objective, and the methods built for annotator disagreement. Each runs on
the corpus of the bias it was built for and, for the length methods, also under formatting,
by a rule fixed before any result was read (Appendix~\ref{app:baselines}). R-DPO \citep{park2024disentangling} and SamPO \citep{lu2024sampo} target
length bias: they are trained on MultiPref and re-scored under formatting without
retraining. R-DPO uses $\alpha = 0.005$, because the published $0.02$ collapses the
policy on this corpus (Appendix~\ref{app:abl_multipref}). Group-DRO
\citep{sagawa2020groupdro} and EM-DPO with MinMax-DPO \citep{chidambaram2024emdpo} target
annotator disagreement. Both are our own implementations (Appendix~\ref{app:baselines}). Group-DRO runs on both corpora, with the
annotator classes as groups on Signed-UltraFeedback and the annotators on MultiPref;
EM-DPO runs on Signed-UltraFeedback and at 0.5B only, since its four rounds over three
type policies cost five times the steps of any other method.

\paragraph{Metrics.}
We evaluate each policy on $300$ held-out prompts, computing five metrics from the generations and held-out judgments (formal definitions in Appendix~\ref{app:readouts}). Each configuration is trained at three seeds, and we report the mean alongside a $95\%$ confidence interval; the reference policy is trained once and therefore has no interval.
\begin{itemize}[nosep,leftmargin=*]
    \item \emph{Attribute rate:} How often the generations carry the declared attribute. On Signed-UltraFeedback, we compute this exactly from the policy's output probabilities rather than from sampling: the probability mass placed on woman-coded ($p(\text{woman})$) or black-coded ($p(\text{black})$) names at the signature position. On MultiPref, this is the mean token count for length bias, and the markdown rate reweighted onto the reference's length distribution for formatting bias, so that a policy cannot lower the measured rate by producing shorter answers alone.
    \item \emph{Bias removed:} The proportion of the gap in attribute rate between standard DPO and the reference policy closed by the method. This is computed per seed against the corresponding DPO seed, scaled such that $0\%$ is DPO and $100\%$ is the reference.
    \item \emph{Held-out gap:} The difference in the policy's implicit reward accuracy between held-out judgment pairs that differ on the attribute, and pairs that do not. A biased policy predicts annotator labels better exactly when the attribute differs; a gap approaching zero, while same-group accuracy holds, shows that the bias has left the implicit reward.
    \item \emph{RM and RM$^-$:} The judge's score, from \texttt{Skywork-Reward-V2-Qwen3-1.7B} \citep{skywork}. \emph{RM$^-$} is the score after uniformly stripping the attribute from all generations (that is, removing signatures in Signed-UltraFeedback or markup in formatting). No length-invariant transform exists for length bias, so there we report the raw RM score and read it with caution.
    \item \emph{KL:} The per-token Kullback-Leibler divergence from the reference, estimated on the policy's own generations. This distinguishes methods that remove the bias by changing the policy from those that stay close to the reference distribution.
\end{itemize}

\subsection{Signed-UltraFeedback}
\label{sec:names_v3}
DPO absorbs the bias on both attributes and on both models (Table~\ref{tab:names_v3}), moving from the
near-balanced reference to $0.96$ to $0.99$. \method{} removes $81$ to $95\%$ of that shift, with both variants and on both
models, at no cost in judge score and at about two thirds of DPO's KL: the policy still differs
from the reference, but not along the attribute.

The two \method{} variants give the same policy to within about $0.01$
everywhere, although \emph{pooled} has one parameter per attribute and
$\bar\theta+\varepsilon_k$ has a mean and sixty deviations. Proposition~\ref{prop:absorb}
predicts this: a deviation depends on who voted, which the policy never sees, so only the
shared mean can reach it. Appendix~\ref{app:abl_variants} tests this prediction with the
annotator identities shuffled and with the annotator classes given. The race-coded attribute is removed
less than the gender-coded one at 0.5B, $81$ against $89\%$, although both were planted
at the same mean. A shared scalar is fitted to how the population votes, not to how it
is biased, and here the classes disagree, so the votes partly cancel: on the pairs that
differ only in the name the attribute-carrying side wins at log-odds $0.70$ against a
realised mean bias of $1.13$, and the learned scalar settles at $0.62$ to $0.68$. The part of
the bias the scalar does not absorb stays in the policy.

\label{sec:votes}
The deviations do not change the policy, but they do estimate each annotator's own bias.
The learned $\theta_k = \bar\theta + \varepsilon_k$ correlate with the planted $\theta_k$ at
$r = 0.95$ on the gender-coded attribute and $0.98$ on the race-coded one, on both models.
When we use them to predict each annotator's held-out votes, they are as accurate as the
planted $\theta_k$ (both as $0.752$ at 0.5B). The pooled scalar gives every annotator the
same bias, so for the class that leans against the race-coded attribute its predictions are
no better than chance (Appendix~\ref{app:votes}).

Neither baseline removes the bias. Group-DRO removes $4$ to $18\%$: it reweights the
groups it is given, and a bias that every group shares is not disagreement between
groups. EM-DPO does not recover the classes, one type ending up with $54$ to $58$ of the
$60$ annotators, and leaves the gender-coded attribute at DPO's level; its $61\%$ on the
race-coded one is averaging rather than removal, since part of the mixture sits on a
type policy carrying the opposite bias.

\begin{table}[t]
\centering
\caption{Signed-UltraFeedback, two attributes and three annotator classes. Columns as defined in Section~\ref{sec:setup}; mean and $95\%$ interval over three seeds. RM$^-$ is comparable within a model only.}
\label{tab:names_v3}
\small
\setlength{\tabcolsep}{3pt}
\resizebox{\linewidth}{!}{%
\begin{tabular}{@{}lcccccc@{}}
\toprule
\textbf{Method} & $p(\text{woman})$ & \textbf{Removed} & $p(\text{black})$ & \textbf{Removed} & \textbf{RM$^-$} & \textbf{KL} \\
\midrule
\multicolumn{7}{@{}l}{\emph{Qwen2.5-0.5B-Instruct, full fine-tuning}} \\
Reference (SFT) & 0.498 & --- & 0.517 & --- & $-3.30$ & --- \\
DPO & $0.964 \pm 0.005$ & --- & $0.991 \pm 0.001$ & --- & $-2.63 \pm 0.28$ & $0.038 \pm 0.001$ \\
\method{} (pooled) & $0.549 \pm 0.008$ & $89 \pm 2\%$ & $0.608 \pm 0.017$ & $81 \pm 4\%$ & $-2.66 \pm 0.29$ & $0.026 \pm 0.001$ \\
\method{} ($\bar\theta+\varepsilon_k$) & $0.542 \pm 0.011$ & $91 \pm 2\%$ & $0.595 \pm 0.018$ & $83 \pm 4\%$ & $-2.76 \pm 0.35$ & $0.026 \pm 0.001$ \\
Group-DRO DPO & $0.927 \pm 0.016$ & $8 \pm 3\%$ & $0.941 \pm 0.018$ & $11 \pm 4\%$ & $-2.62 \pm 0.08$ & $0.035 \pm 0.001$ \\
EM-DPO + MinMax-DPO & $0.923 \pm 0.109$ & $9 \pm 24\%$ & $0.702 \pm 0.157$ & $61 \pm 33\%$ & $-2.63 \pm 0.10$ & $0.041 \pm 0.007$ \\
\midrule
\multicolumn{7}{@{}l}{\emph{Llama-3.1-8B-Instruct, LoRA}} \\
Reference (SFT) & 0.465 & --- & 0.521 & --- & $0.41$ & --- \\
DPO & $0.967 \pm 0.012$ & --- & $0.986 \pm 0.005$ & --- & $1.43 \pm 0.17$ & $0.021 \pm 0.002$ \\
\method{} (pooled) & $0.492 \pm 0.002$ & $95 \pm {<}1\%$ & $0.555 \pm 0.002$ & $93 \pm {<}1\%$ & $1.41 \pm 0.26$ & $0.013 \pm 0.003$ \\
\method{} ($\bar\theta+\varepsilon_k$) & $0.491 \pm 0.002$ & $95 \pm {<}1\%$ & $0.552 \pm 0.002$ & $93 \pm {<}1\%$ & $1.49 \pm 0.12$ & $0.014 \pm 0.002$ \\
Group-DRO DPO & $0.945 \pm 0.018$ & $4 \pm 2\%$ & $0.904 \pm 0.020$ & $18 \pm 3\%$ & $1.44 \pm 0.25$ & $0.020 \pm 0.003$ \\
\bottomrule
\end{tabular}}
\end{table}

\subsection{MultiPref: length bias}
\label{sec:multipref}
On MultiPref the annotators are real and length has no natural target, since a longer
answer is sometimes the better one. The token count therefore says how far the policy
moved, not whether it stopped in the right place. We read the held-out gap for that: it
compares how well the policy's implicit reward predicts these annotators' labels on
held-out pairs that differ on the declared attribute with how well it predicts them on
pairs that do not.

Table~\ref{tab:multipref} reports the results. Under DPO the policy produces answers
$96$ tokens longer than the reference at 0.5B and $68$ tokens longer at 8B, with
held-out gaps of $0.13$ and $0.09$. \method{} removes $45$ to $51\%$ of that increase
and reduces the gap to $0.02$ and $0.00$ at 0.5B and to $0.03$ and $0.05$ at 8B. Its
learned scalar converges to $0.77$ to $0.81$ at 0.5B and to $0.69$ to $0.71$ at 8B,
below the offline estimate of $0.99$, which counts every win of the longer response as
bias.

Each baseline misses the bias in a different way. R-DPO removes more tokens than
\method{} does, $77$ to $89\%$ of DPO's increase, but its gap becomes negative, $-0.04$
at both scales: the policy now prefers the shorter response on the pairs that carry the
attribute, which inverts the bias rather than removing it. Under SamPO the policy is
longer than under DPO ($-38\%$ and $-7\%$), and Group-DRO changes neither the length nor
the gap, since a bias that every annotator shares is not disagreement between
annotators. At 8B the judge ranks the methods in order of their length, from $4.9$ for
the reference to $6.7$ for SamPO, so \method{}'s $6.1$ against DPO's $6.5$ cannot be
separated from the tokens removed. The held-out gap does not use the judge, which is why
we report it here.

\begin{table}[t]
\centering
\caption{MultiPref with length declared. Columns as in Section~\ref{sec:setup}; mean and $95\%$ interval over three seeds; RM is the raw judge score, which rises with answer length (Section~\ref{sec:multipref}).}
\label{tab:multipref}
\small
\setlength{\tabcolsep}{5pt}
\begin{tabular}{@{}lccccc@{}}
\toprule
\textbf{Method} & \textbf{Tokens} & \textbf{Removed} & \textbf{Gap} & \textbf{RM} & \textbf{KL} \\
\midrule
\multicolumn{6}{@{}l}{\emph{Qwen2.5-0.5B-Instruct, full fine-tuning}} \\
Reference (SFT) & 288.6 & --- & --- & $-1.34$ & --- \\
DPO & $384.5 \pm 3.7$ & --- & $0.126 \pm 0.022$ & $-0.66 \pm 0.16$ & $0.033 \pm {<}0.001$ \\
\method{} (pooled) & $341.3 \pm 8.7$ & $45 \pm 10\%$ & $0.019 \pm 0.045$ & $-0.66 \pm 0.20$ & $0.033 \pm 0.001$ \\
\method{} ($\bar\theta+\varepsilon_k$) & $335.3 \pm 7.4$ & $51 \pm 8\%$ & $0.001 \pm 0.032$ & $-0.78 \pm 0.30$ & $0.032 \pm 0.001$ \\
R-DPO ($\alpha = 0.005$) & $310.8 \pm 10.3$ & $77 \pm 10\%$ & $-0.040 \pm 0.007$ & $-0.74 \pm 0.38$ & $0.035 \pm 0.002$ \\
SamPO & $420.9 \pm 4.4$ & $-38 \pm 8\%$ & $0.088 \pm 0.023$ & $-0.76 \pm 0.07$ & $0.044 \pm 0.002$ \\
Group-DRO DPO & $389.9 \pm 4.3$ & $-6 \pm 3\%$ & $0.122 \pm 0.032$ & $-0.53 \pm 0.24$ & $0.032 \pm 0.002$ \\
\midrule
\multicolumn{6}{@{}l}{\emph{Llama-3.1-8B-Instruct, LoRA}} \\
Reference (SFT) & 309.7 & --- & --- & $4.85$ & --- \\
DPO & $378.2 \pm 4.7$ & --- & $0.095 \pm 0.008$ & $6.52 \pm 0.11$ & $0.018 \pm 0.001$ \\
\method{} (pooled) & $343.9 \pm 2.6$ & $50 \pm 4\%$ & $0.034 \pm 0.032$ & $6.10 \pm 0.02$ & $0.018 \pm {<}0.001$ \\
\method{} ($\bar\theta+\varepsilon_k$) & $346.9 \pm 6.5$ & $46 \pm 6\%$ & $0.046 \pm 0.013$ & $6.10 \pm 0.09$ & $0.018 \pm {<}0.001$ \\
R-DPO ($\alpha = 0.005$) & $317.4 \pm 2.9$ & $89 \pm 4\%$ & $-0.038 \pm 0.004$ & $5.73 \pm 0.17$ & $0.020 \pm 0.001$ \\
SamPO & $382.8 \pm 3.1$ & $-7 \pm 9\%$ & $0.061 \pm 0.024$ & $6.70 \pm 0.54$ & $0.022 \pm {<}0.001$ \\
Group-DRO DPO & $381.4 \pm 12.1$ & $-5 \pm 11\%$ & $0.097 \pm 0.016$ & $6.49 \pm 0.34$ & $0.018 \pm 0.001$ \\
\bottomrule
\end{tabular}
\end{table}

\subsection{MultiPref: formatting bias}
\label{sec:formatting}
The data is the same as in Section~\ref{sec:multipref}, with formatting declared in place
of length: the same annotators and the same judgments. The two attributes are favoured
about equally in the labels, markdown winning at log-odds $0.98$ and length at $0.99$, but
DPO amplifies only one of them. Under length it adds $96$ tokens, many times the spread
across seeds, while under formatting it raises the markdown rate by $0.035$, less than the
$\pm 0.053$ interval on that rate. The rate is therefore a weak readout here, since DPO
barely moves it, and we read the held-out gap instead.

DPO opens a gap of $0.066$ (Table~\ref{tab:formatting}). \method{} closes it to $0.004$
and $0.000$ while leaving the markdown rate where the reference had it. At 8B DPO does
not raise the rate beyond seed noise ($0.567 \pm 0.034$ against $0.560$), and still opens a
gap of $0.067$, which both \method{} variants halve. %  to $0.033$.

Neither length method can be given another attribute, and correcting length does not
correct formatting: R-DPO turns the gap negative again ($-0.039$) and SamPO leaves it near
DPO's level ($0.053$).
At its published strength, $\alpha = 0.02$, R-DPO raises the markdown rate to $0.70$,
above DPO, in answers collapsed to $138$ tokens
(Appendix~\ref{app:abl_multipref}): the pressure DPO had put on length has moved to
markdown. On two attributes of one corpus this is the substitution pattern of
\citet{lamparth2026substitution}, and it is what declaring the attribute avoids.
Appendix~\ref{app:abl_variants} shows the same pattern inside Signed-UltraFeedback when
only one of its two biased attributes is declared.

\begin{table}[t]
\centering
\caption{MultiPref with formatting declared, on the judgments and annotators of Table~\ref{tab:multipref}. Columns as in Table~\ref{tab:multipref}. No removal share is given because DPO moves the markdown rate by less than its seed interval (Section~\ref{sec:formatting}). The length baselines are their Table~\ref{tab:multipref} checkpoints scored under formatting without retraining.}
\label{tab:formatting}
\small
\setlength{\tabcolsep}{5pt}
\begin{tabular}{@{}lcccc@{}}
\toprule
\textbf{Method} & \textbf{Rate} & \textbf{Gap} & \textbf{RM$^-$} & \textbf{KL} \\
\midrule
\multicolumn{5}{@{}l}{\emph{Qwen2.5-0.5B-Instruct, full fine-tuning}} \\
Reference (SFT) & 0.583 & --- & $-1.51$ & --- \\
DPO & $0.618 \pm 0.053$ & $0.066 \pm 0.025$ & $-0.98 \pm 0.14$ & $0.033 \pm {<}0.001$ \\
\method{} (pooled) & $0.592 \pm 0.027$ & $0.004 \pm 0.013$ & $-0.89 \pm 0.12$ & $0.032 \pm 0.002$ \\
\method{} ($\bar\theta+\varepsilon_k$) & $0.581 \pm 0.034$ & $0.000 \pm 0.051$ & $-0.86 \pm 0.19$ & $0.032 \pm 0.002$ \\
R-DPO ($\alpha = 0.005$, re-scored) & $0.634 \pm 0.003$ & $-0.039 \pm 0.039$ & $-0.98 \pm 0.34$ & $0.035 \pm 0.002$ \\
SamPO (re-scored) & $0.636 \pm 0.049$ & $0.053 \pm 0.024$ & $-1.16 \pm 0.15$ & $0.044 \pm 0.002$ \\
\midrule
\multicolumn{5}{@{}l}{\emph{Llama-3.1-8B-Instruct, LoRA}} \\
Reference (SFT) & 0.560 & --- & $4.67$ & --- \\
DPO & $0.567 \pm 0.034$ & $0.067 \pm 0.016$ & $6.18 \pm 0.12$ & $0.018 \pm 0.001$ \\
\method{} (pooled) & $0.536 \pm 0.018$ & $0.033 \pm 0.033$ & $6.05 \pm 0.03$ & $0.018 \pm 0.001$ \\
\method{} ($\bar\theta+\varepsilon_k$) & $0.539 \pm 0.022$ & $0.033 \pm 0.046$ & $6.08 \pm 0.09$ & $0.018 \pm {<}0.001$ \\
R-DPO ($\alpha = 0.005$, re-scored) & $0.568 \pm 0.041$ & $-0.031 \pm 0.029$ & $5.50 \pm 0.15$ & $0.020 \pm 0.001$ \\
SamPO (re-scored) & $0.569 \pm 0.013$ & $0.040 \pm 0.007$ & $6.34 \pm 0.54$ & $0.022 \pm {<}0.001$ \\
\bottomrule
\end{tabular}
\end{table}

\section{Conclusions}
\label{sec:conclusion}

We addressed the vulnerability of preference optimization to a bias that annotators share by
introducing Bias-Adjusted DPO (\method{}), a generalization of DPO that adds one bias parameter
per annotator for each declared attribute. We proved that the objective is convex in these
parameters and that the votes fix every annotator's bias relative to the others. How much of the
shared bias counts as quality is left to one free choice, which sets the aligned policy's attribute
rate. On three models, \method{} removed $81$ to $95\%$ of the shift DPO absorbs on planted biases
and about half of the length DPO adds on MultiPref, with no higher KL to the reference than DPO and
no loss in judged quality wherever the judge can be read independently of the attribute.

\paragraph{Limitations and future work.}
\method{} removes only the bias it is told about: the attribute must be declared and computable
from the response, and an attribute that is not declared stays in the policy
(Appendix~\ref{app:abl_variants}). Because the labels do not fix the shared level of the bias,
the default anchor keeps the reference's attribute rate, including any bias the reference already
has; a target rate removes it, but the target must be chosen. That level could instead be estimated
from a small set of judgments by annotators known to be unbiased. Our evidence comes from one
reward-model judge, one corpus with real annotators whose two attributes are both surface features,
and LoRA at 7B and 8B; a corpus with annotator identities and a social attribute, such as a persona
or a language variety, is the natural next test. Since the bias term is additive to the reward, it
also applies to reward-model training and to other pairwise objectives such as IPO \citep{ipo}, and
the per-annotator estimates could flag biased annotators while a pool is still being labelled.

\subsection*{AI use statement}
In this work, we used generative AI tools (Claude Code with Fable 5.1 and Opus 5) as a coding assistant and to polish the
writing of the manuscript. We have reviewed all AI-assisted code and text, and we take full responsibility for the final
content of this work.

\subsection*{Ethics statement}
This work aims to reduce the influence of systematic annotator bias on aligned language
models. Two cautions apply. First, the partition of response attributes into bias to be
removed and quality to be preserved is a declared input, since preference data determine the biases only up to a common offset; the choice is normative and should involve the stakeholders
affected. Second, per-annotator modelling requires annotator identifiers and produces
per-person bias estimates; we recommend treating both as personal data, and we use only
datasets whose identifiers were pseudonymised by their publishers.

\subsection*{Reproducibility statement}
All models (\texttt{Qwen2.5-\allowbreak 0.5B-\allowbreak Instruct},
\texttt{Llama-\allowbreak 3.1-\allowbreak 8B-\allowbreak Instruct} and
\texttt{Mistral-\allowbreak 7B-\allowbreak Instruct-\allowbreak v0.3}, with
\texttt{Skywork-\allowbreak Reward-\allowbreak V2-\allowbreak Qwen3-\allowbreak 1.7B} as
the judge) and datasets (\texttt{ultrafeedback\_binarized} and
\texttt{allenai/multipref}) are public. Every reported number derives from a logged run keyed by
(method, parameters, seed); every cell of the main tables is a mean over seeds 42, 123
and 456 with a 95\% interval, and single-seed appendix rows are labelled as
controls. Complete data-generation procedures and
hyperparameters are in Appendices~\ref{app:multipref} to~\ref{app:training}. All runs
used a single NVIDIA A100 80GB GPU; the peak memory and duration of each kind of run are
given in Appendix~\ref{app:training}. The source code, with the configuration of every experiment and the commands that
reproduce each table, is available at \url{https://github.com/Ambress92/BA-DPO}.

\bibliographystyle{iclr2027_conference}
\bibliography{references}

\appendix
\addtocontents{toc}{\protect\setcounter{tocdepth}{2}}
\renewcommand{\contentsname}{Appendix contents}
\tableofcontents
\section{Related Work}
\label{sec:related}

\paragraph{Length bias in RLHF and DPO.}
A substantial line of work identifies the tendency of preference-trained models to
exploit response length as a proxy for quality. \citet{park2024disentangling} study
length exploitation in DPO specifically and propose a regularisation strategy,
\citet{lu2024sampo} eliminate length reliance through down-sampled KL divergence, and
\citet{liu2024ld} desensitise DPO to length by decoupling explicit length preference
from implicit quality signals. \citet{wang2024removing} show that length is not the only
such bias: reward models also favour the response format that the prompt template
invites, and correcting length alone leaves that preference in place. On the reward-model side, \citet{zhao2025fimi} learn and correct length-bias patterns before preference optimisation, while \citet{shen2023loose} and \citet{chen2024odin} give the reward model a separate head for length and discard it when the policy is trained. All of them build length into the objective; \method{} treats length as one instance of a per-annotator additive bias on an attribute the user declares, and our experiments show that correcting length leaves a second attribute in place.

\paragraph{Format and other spurious features.}
Beyond length, \citet{zhang2024format} show that format features such as lists, markdown
and emojis bias alignment in the same way, and \citet{wang2025causal} propose causal
reward models addressing several bias types jointly. \citet{ng2025debiasing} learn a representation that
separates the factors the reward should depend on from the spurious ones, with
identifiability guarantees. These operate at the reward-model level, whereas \method{} corrects at the preference-likelihood level and so remains compatible with the reward-model-free setting DPO defines. \citet{biasdpo} address a different problem under a similar name: they use DPO to make a model produce less biased text, with a curated dataset in which the less biased of two completions is always the preferred one. There the bias is in the model's outputs and the labels are the correction; here the bias is in the annotators' labels and the outputs are what it corrupts. Their preferred completions carry the quality signal to be learned, whereas our declared attribute carries the bias to be removed.

\paragraph{Bias-only models.}
Training a separate term on a declared biased feature alongside the main model and
dropping it at test time is the standard way to keep a classifier from relying on a known
dataset bias \citep{clark2019ensemble, he2019unlearn, mahabadi2020end}. \method{} is that
construction inside preference optimisation: the bias term is additive in the reward, so it
survives the DPO reparameterisation and needs no second network, and its coefficient is per
annotator, which those methods have no analogue for because their corpora carry no
annotator identity.

\paragraph{Annotator heterogeneity.}
Standard RLHF treats all annotators as sampling from one latent reward.
\citet{xiao2024preference} analyse the algorithmic bias this induces and show that RLHF
can exhibit preference collapse under annotator disagreement. In the taxonomy of
\citet{liu2025survey}, \method{} belongs to the data-quality heterogeneity branch,
alongside methods that model disagreement as a mixture over latent preference types
\citep{chidambaram2024emdpo} or as preference dispersion \citep{mallowspo};
\citet{chhan2024crowd} use annotator identifiers to infer per-user reliability in preference-based reward learning, though not within DPO. ADPO \citep{adpo} anchors the policy's logits to the reference so that a response's prior popularity is separated from its quality; the offset it removes is common to every response and carries no annotator identity, so it does not address a preference for an attribute that only some responses carry. Those methods represent or balance disagreement, whereas \method{} models it as an additive bias on a declared attribute, which is what makes removal rather than accommodation possible, and separates what its parameters do: the mean removes the bias and the deviations describe the annotators.

\section{Proofs}
\label{app:proofs}

Throughout, judgments are indexed by $j = 1, \dots, N$; judgment $j$ was made by
annotator $k(j)$ on prompt $x_j$, with winner $y_{w,j}$ and loser $y_{l,j}$, and
$\Delta\delta_j \coloneqq \delta_{g(x_j,y_{w,j})} - \delta_{g(x_j,y_{l,j})} \in
\{-1, 0, 1\}$. We write
\[
    \phi(t) \coloneqq -\log\sigma(t) = \log\bigl(1 + e^{-t}\bigr),
    \qquad
    \phi'(t) = -\sigma(-t),
    \qquad
    \phi''(t) = \sigma(t)\,\sigma(-t) > 0,
\]
so that the negative log-likelihood of any model whose logit on judgment $j$ is $L_j$
equals $\sum_j \phi(L_j)$.

\subsection{Proof of Proposition~\ref{prop:compat}}

\propcompat

Since $Z(x)$ is finite and positive, $\pi(y \mid x) = \piref(y \mid x)\exp(r(x,y)/\beta)
/ Z(x)$ is a probability distribution over responses for every prompt. Taking logarithms
and solving for the reward,
\[
    r(x,y) = \beta\log\frac{\pi(y \mid x)}{\piref(y \mid x)} + \beta\log Z(x),
\]
which is the reparameterisation of Equation~5 of \citet{dpo}.

Now substitute this expression for $r$ into the logit of the biased Bradley-Terry model
\eqref{eq:biased_bt}:
\begin{align*}
    &r(x,y_1) + \theta_k\,\delta_{g(x,y_1)} - r(x,y_2) - \theta_k\,\delta_{g(x,y_2)} \\
    &\quad = \beta\log\frac{\pi(y_1 \mid x)}{\piref(y_1 \mid x)} + \beta\log Z(x)
      - \beta\log\frac{\pi(y_2 \mid x)}{\piref(y_2 \mid x)} - \beta\log Z(x)
      + \theta_k\bigl(\delta_{g(x,y_1)} - \delta_{g(x,y_2)}\bigr) \\
    &\quad = \beta\log\frac{\pi(y_1 \mid x)}{\piref(y_1 \mid x)}
      - \beta\log\frac{\pi(y_2 \mid x)}{\piref(y_2 \mid x)}
      + \theta_k\bigl(\delta_{g(x,y_1)} - \delta_{g(x,y_2)}\bigr).
\end{align*}
The two copies of $\beta\log Z(x)$ cancel because both responses share the prompt, and
the bias terms pass through untouched because they are added to the reward and do not
involve $Z(x)$. Applying $\sigma$ to the last line gives the stated form. A bias that
entered multiplicatively, or through the partition function, would not cancel in this
way. \qed

\subsection{Proof of Theorem~\ref{thm:identif}}

\thmidentif

We first write the logit. For judgment $j$ the logit of the bias-aware model is
\[
    L_j(r, \theta) = r(x_j, y_{w,j}) - r(x_j, y_{l,j}) + \theta_{k(j)}\,\Delta\delta_j .
\]
As a function of the reward values and the bias vector together, $L_j$ is a sum of
parameters with coefficients $+1$, $-1$ and $\Delta\delta_j$, hence affine. For a fixed
policy the reward difference is a constant $u_j$ and $L_j = u_j + \theta_{k(j)}
\Delta\delta_j$ is affine in $\theta$ alone. The negative log-likelihood is
\[
    \mathcal{L} = \sum_{j=1}^N \phi\bigl(L_j\bigr).
\]

Convexity follows. $\phi$ is convex because $\phi'' > 0$. A convex function composed with an affine map is
convex, and a sum of convex functions is convex. Hence $\mathcal{L}$ is convex in
$(r, \theta)$ jointly for an arbitrary reward, and in particular convex in $\theta$ for
a fixed policy.

All minimisers give the same probabilities. Let $(r, \theta)$ and $(r', \theta')$ both
minimise $\mathcal{L}$, with logits $L_j$ and $L_j'$. Their midpoint has logits
$\tfrac12(L_j + L_j')$, because $L_j$ is affine, and by convexity it minimises
$\mathcal{L}$ as well. Since $\phi$ is strictly convex ($\phi'' > 0$),
\[
    \sum_j \phi\Bigl(\tfrac{L_j + L_j'}{2}\Bigr)
    \;\le\; \tfrac12 \sum_j \phi(L_j) + \tfrac12 \sum_j \phi(L_j'),
\]
with equality if and only if $L_j = L_j'$ for every $j$. All three quantities equal the
minimum of $\mathcal{L}$, so equality holds, the logits agree, and so do the
probabilities $\sigma(L_j)$.

For the strictness of the dependence on a single $\theta_k$, take the direction that moves only $\theta_k$. Since
$\partial L_j / \partial\theta_k = \Delta\delta_j$ when $k(j) = k$ and $0$ otherwise,
the second derivative of $\mathcal{L}$ along this direction is
\[
    \frac{\partial^2 \mathcal{L}}{\partial\theta_k^2}
    = \sum_{j:\,k(j) = k} \phi''(L_j)\,\Delta\delta_j^2 .
\]
Every term is non-negative, and a term is positive exactly when $\Delta\delta_j \neq
0$, that is, when judgment $j$ is a cross-group comparison. So the second derivative is
positive at every point as soon as annotator $k$ has one cross-group judgment, which is
strict convexity in $\theta_k$.

For the characterisation of the minimisers, let $(r, \theta)$ and $(r', \theta')$ be two parameter vectors and write
$dr \coloneqq r' - r$ and $d\theta \coloneqq \theta' - \theta$ for their difference.
The probability assigned to judgment $j$ is $\sigma(L_j)$, and $\sigma$ is strictly
increasing, so the two vectors assign the same probability to every judgment if and
only if $L_j(r', \theta') = L_j(r, \theta)$ for every $j$. Because $L_j$ is affine, this
is the linear system
\begin{equation}\label{eq:null_system}
    dr(x_j, y_{w,j}) - dr(x_j, y_{l,j}) + d\theta_{k(j)}\,\Delta\delta_j = 0
    \qquad \text{for all } j = 1, \dots, N .
\end{equation}
It remains to show that the solutions of \eqref{eq:null_system} are
exactly the pairs $(dr, d\theta)$ with $d\theta = -c\mathbf{1}$ and $dr - c\,\delta_g$
constant on the two responses of every compared pair, for some $c \in \R$.

We solve the system one compared pair at a time. Fix a compared pair $(y_1, y_2)$ on prompt $x$ and let $K$ be the set of annotators
who judged it. If the pair is same-group, then $\Delta\delta_j = 0$ for each of these
judgments and \eqref{eq:null_system} reads
\[
    dr(x, y_1) = dr(x, y_2).
\]
If the pair is cross-group, then for each annotator $k \in K$, whichever response that
annotator chose, \eqref{eq:null_system} reads
\begin{equation}\label{eq:cross_pair}
    dr(x, y_1) - dr(x, y_2) = -\,d\theta_k\,\bigl(\delta_{g(x,y_1)} - \delta_{g(x,y_2)}\bigr).
\end{equation}
(If the annotator chose $y_2$, both sides of \eqref{eq:null_system} are the negatives
of those in \eqref{eq:cross_pair}, so the equation is the same.) The left side of
\eqref{eq:cross_pair} does not depend on $k$, and the bracket on the right is $\pm 1$,
so $d\theta_k$ takes the same value for every $k \in K$: annotators who judged the same
cross-group pair have the same $d\theta_k$.

Connectedness spreads this common value. Two annotators joined by an edge of the graph,
that is, who judged a cross-group pair in common, share the same $d\theta$. When the pool is
connected, any two annotators are joined by a chain of such edges, so the value is the
same along the chain and hence for the whole pool: there is a single constant $c$ with
\[
    d\theta_k = -c \quad \text{for every } k, \qquad \text{that is,} \qquad
    d\theta = -c\mathbf{1}.
\]

Next the reward part. Substitute $d\theta_k = -c$ back into \eqref{eq:cross_pair}: every cross-group pair
satisfies
\[
    dr(x, y_1) - dr(x, y_2) = c\,\bigl(\delta_{g(x,y_1)} - \delta_{g(x,y_2)}\bigr),
    \qquad \text{that is,} \qquad
    \bigl(dr - c\,\delta_g\bigr)(x, y_1) = \bigl(dr - c\,\delta_g\bigr)(x, y_2).
\]
Every same-group pair satisfies the same equality, because there $dr(x, y_1) = dr(x,
y_2)$ and $\delta_{g(x,y_1)} = \delta_{g(x,y_2)}$. So $dr - c\,\delta_g$ is constant on
the two responses of every compared pair. This proves that every solution of
\eqref{eq:null_system} has the stated form.

Conversely, take any $c$, any $d\theta = -c\mathbf{1}$, and any $dr$ such that $dr -
c\,\delta_g$ is constant within every compared pair. For judgment $j$,
\[
    dr(x_j, y_{w,j}) - dr(x_j, y_{l,j})
    = c\,\bigl(\delta_{g(x_j,y_{w,j})} - \delta_{g(x_j,y_{l,j})}\bigr)
    = c\,\Delta\delta_j
    = -\,d\theta_{k(j)}\,\Delta\delta_j ,
\]
so \eqref{eq:null_system} holds, and the solutions are exactly as claimed.

It remains to put the two halves together. Any two minimisers have the same logits, so
their difference solves \eqref{eq:null_system}, and conversely every solution leaves all
logits and hence $\mathcal{L}$ unchanged. The minimisers are therefore exactly the stated
family. Any two of them have $d\theta = -c\mathbf{1}$, so $d\theta_k - d\theta_{k'} = 0$:
they agree on every difference $\theta_k - \theta_{k'}$, while the common level of $\theta$
shifts by $c$. \qed

\subsection{Proof of Corollary~\ref{prop:degeneracy}}

\cordegeneracy

The transformation $r \mapsto r + c\,\delta_g$, $\theta \mapsto \theta - c\mathbf{1}$ is
the case of Theorem~\ref{thm:identif} in which the within-pair-constant part of the
reward shift is zero, so the invariance follows from the theorem. The direct check is
short. The logit for a comparison $(y_w, y_l)$ by annotator $k$ is
\[
    L = r(x, y_w) - r(x, y_l) + \theta_k\bigl(\delta_{g(x,y_w)} - \delta_{g(x,y_l)}\bigr).
\]
After the transformation it becomes
\begin{align*}
    L' &= \bigl(r(x, y_w) + c\,\delta_{g(x,y_w)}\bigr) - \bigl(r(x, y_l) + c\,\delta_{g(x,y_l)}\bigr)
        + (\theta_k - c)\bigl(\delta_{g(x,y_w)} - \delta_{g(x,y_l)}\bigr) \\
       &= r(x, y_w) - r(x, y_l)
        + c\bigl(\delta_{g(x,y_w)} - \delta_{g(x,y_l)}\bigr)
        - c\bigl(\delta_{g(x,y_w)} - \delta_{g(x,y_l)}\bigr)
        + \theta_k\bigl(\delta_{g(x,y_w)} - \delta_{g(x,y_l)}\bigr) \\
       &= L .
\end{align*}
Every logit, and so every likelihood term, is unchanged, so the data cannot distinguish
$(r, \bar\theta)$ from $(r + c\,\delta_g, \bar\theta - c)$.

For the deviations, Theorem~\ref{thm:identif} shows that every difference
$\theta_k - \theta_{k'}$ takes the same value at every minimiser. Each deviation is an
average of such differences,
\[
    \theta_k - \bar\theta = \frac{1}{m}\sum_{k'=1}^m \bigl(\theta_k - \theta_{k'}\bigr),
\]
so each $\varepsilon_k = \theta_k - \bar\theta$ does too. \qed

\subsection{Proof of Proposition~\ref{prop:absorb}}

\propabsorb

Write $\theta_k = \bar{\theta} + \varepsilon_k$ with $\sum_k \varepsilon_k = 0$. The
bias contribution to the logit of a comparison splits into two terms,
\[
    \theta_k\bigl(\delta_{g(x,y_w)} - \delta_{g(x,y_l)}\bigr)
    = \underbrace{\bar{\theta}\bigl(\delta_{g(x,y_w)} - \delta_{g(x,y_l)}\bigr)}_{\text{does not depend on } k}
    + \underbrace{\varepsilon_k\bigl(\delta_{g(x,y_w)} - \delta_{g(x,y_l)}\bigr)}_{\text{depends on } k}.
\]

The first term can be moved into the reward. Define
\[
    \tilde{r}(x,y) \coloneqq r(x,y) + \bar{\theta}\,\delta_{g(x,y)} .
\]
Since $\delta_{g(x,y)}$ is a function of the prompt and the response only, $\tilde r$
is a reward function on $(x,y)$ like any other, and the reward difference of $\tilde r$
on a comparison equals the reward difference of $r$ plus the first term. By
Proposition~\ref{prop:compat}, $\tilde r$ admits a policy representation, so the shared
component of the bias can be carried by the policy's implicit reward.

The second term cannot. Suppose there were a function $h(x,y)$ with
\[
    \varepsilon_k\,\delta_{g(x,y)} = h(x,y) \qquad \text{for all } k \text{ and all } (x,y).
\]
Pick any $(x,y)$ with $\delta_{g(x,y)} = 1$; one exists whenever the attribute occurs
in the corpus. Then $\varepsilon_k = h(x,y)$ for every $k$, so all deviations are equal
to one number, and since they sum to zero that number is $0$. Hence the second term is
a function of $(x,y)$ only when every $\varepsilon_k = 0$, that is, when all annotators
have the same bias. Otherwise no policy $\pitheta(y \mid x)$, whose implicit reward is
a function of $(x,y)$, can represent it. \qed

\subsection{The speed of the mean under the two parameterisations}
\label{app:proof_rate}

\begin{restatable}[The mean learns $m$ times more slowly]{proposition}{proprate}\label{prop:rate}
Assume annotators contribute equal data shares and gradient descent with a common
learning rate on the bias parameters. Under the shared-mean parameterisation $\theta_k =
\bar\theta + \varepsilon_k$, one step moves the shared parameter $\bar\theta$ exactly $m$
times further than the same step moves the mean of the $\theta_k$ under the naive
per-annotator model, and moves the deviations from the mean identically.
\end{restatable}

By \eqref{eq:grad_bias}, the gradient of the loss with respect to $\theta_k$ is
\[
    \frac{\partial \mathcal{L}}{\partial \theta_k}
    = \frac{n_k}{N}\,G_k ,
    \qquad
    G_k \coloneqq -\E_{(x,y_w,y_l)\sim\mathcal{D}_k}\Bigl[\sigma(-u-b_k)\,\Delta\delta\Bigr],
\]
and with equal shares $n_k = N/m$ this is $G_k / m$.

In the naive per-annotator model each $\theta_k$ is its own parameter, so one
gradient-descent step with learning rate $\eta$ moves it by $-\eta\,G_k/m$, and the
mean of the parameters moves by
\[
    \Delta\bar\theta
    = \frac{1}{m}\sum_{k=1}^m \Bigl(-\eta\,\frac{G_k}{m}\Bigr)
    = -\frac{\eta}{m}\,\overline{G},
    \qquad
    \overline{G} \coloneqq \frac{1}{m}\sum_{k=1}^m G_k .
\]

In the shared-mean model $\theta_k = \bar\theta + \varepsilon_k$, the parameter
$\bar\theta$ enters every $\theta_k$ with coefficient $1$, so by the chain rule
\[
    \frac{\partial \mathcal{L}}{\partial \bar\theta}
    = \sum_{k=1}^m \frac{\partial \mathcal{L}}{\partial \theta_k}
    = \sum_{k=1}^m \frac{G_k}{m}
    = \overline{G},
\]
and one step moves $\bar\theta$ by $-\eta\,\overline{G}$, which is $m$ times the
motion of the naive mean.

For the deviations, the gradient on $\varepsilon_k$ equals the gradient on $\theta_k$,
namely $G_k/m$, because $\varepsilon_k$ enters only $\theta_k$ and with coefficient
$1$. In the naive model the deviation of $\theta_k$ from the mean of the $\theta$'s
therefore moves by
\[
    -\eta\,\frac{G_k}{m} - \Delta\bar\theta = -\frac{\eta}{m}\bigl(G_k - \overline{G}\bigr).
\]
In the shared-mean model, $\theta_k$ moves by $-\eta\,\overline{G} - \eta\,G_k/m$ and
the mean of the $\theta$'s by $-\eta\,\overline{G} - \eta\,\overline{G}/m$, so the
deviation moves by the same amount, $-\frac{\eta}{m}(G_k - \overline{G})$. The
constraint $\sum_k \varepsilon_k = 0$ of Proposition~\ref{prop:absorb} is a
decomposition, not a constraint enforced in training; the mean of the $\varepsilon_k$
drifts by $-\eta\,\overline{G}/m$ per step, the same $m$-times smaller motion as the
naive mean. \qed

\paragraph{A rough estimate under Adam.}
The bias parameters are trained with Adam, which divides each coordinate's step by a
running estimate of the magnitude of its gradient. Consider a coordinate whose gradient
in a step equals $g \neq 0$ with probability $p$ and $0$ otherwise, independently across
steps, and Adam with learning rate $\eta$, moment parameters $\beta_1, \beta_2$ and a
negligible stabiliser $\epsilon$. The first-moment estimate $m_t$ is an exponential
average of past gradients, so at stationarity $\E[m_t] = p\,g$. The second-moment
estimate $v_t$ averages past squared gradients over a window of about $1/(1-\beta_2)$
steps, which for the default $\beta_2 = 0.999$ makes it close to its mean $p\,g^2$. The
expected step is therefore
\[
    -\eta\,\frac{\E[m_t]}{\sqrt{p\,g^2}} = -\eta\,\sqrt{p}\,\operatorname{sign}(g),
\]
against $-\eta\,p\,g$ under gradient descent: a coordinate updated in a fraction $p$ of
the steps moves at a rate proportional to $\sqrt{p}$ rather than $p$. In a step of $B$
judgments of which $b$ are cross-group, a free $\theta_k$ receives a gradient only from
its own annotator's cross-group judgments, so $p_k \approx b/m$ when the judgments are
spread evenly over the $m$ annotators, whereas $\bar\theta$ receives a gradient whenever
the step contains any cross-group judgment, so $p \approx 1$. The shared parameter
therefore moves about $\sqrt{m/b}$ times further per step than each free $\theta_k$, and
than their mean, in place of the factor $m$ of Proposition~\ref{prop:rate}. In our runs
$B = 32$; on MultiPref, with $227$ annotators and $59\%$ of the judgments cross-group,
the factor is about $3.5$ against $227$ under gradient descent. This is an estimate, not a
proven result: it takes the gradient as constant when present and does not model the
policy's own motion, so it orders the two parameterisations and does not predict where the
naive mean stops. The measured runs answer that: the free mean ends at $0.50$ of $\bar\theta$
under formatting and at $0.61$ under length (Appendix~\ref{app:abl_multipref}).

\section{The common offset and the choice of anchor}
\label{app:anchor}

Theorem~\ref{thm:identif} and Corollary~\ref{prop:degeneracy} state that the labels determine the annotators' biases up to a common offset. This appendix says what that offset is in terms of the policy, why the
reference policy fixes it in every experiment of the paper, and how a requirement on the
outputs can fix it instead. Nothing here adds an assumption: the results are those of
Section~\ref{sec:theory}, read in a coordinate in which the offset is visible.
We write the argument for one binary attribute; with several attributes it applies to each
coordinate of $\theta_k$ separately, and the swap pairs of Signed-UltraFeedback, which flip
exactly one attribute, are the pairs it refers to.

\subsection{The policy's tilt toward the attribute}
\label{app:anchor_tilt}

By Proposition~\ref{prop:compat}, the policy enters the likelihood only through its implicit
reward
\[
    r(x,y) \coloneqq \beta \log \frac{\pi(y \mid x)}{\piref(y \mid x)},
\]
a score the policy assigns to every response relative to the reference: positive where the
policy has made the response more likely than the reference did, negative where less. The
logit of a judgment by annotator $k$ is $r(x,y_w) - r(x,y_l) + \theta_k\,\Delta\delta$ with
$\Delta\delta = \delta_{g(x,y_w)} - \delta_{g(x,y_l)}$.

For a response $y$ that carries the attribute, $\delta_{g(x,y)} = 1$, let $\bar y$ denote its
counterfactual: the same response with the attribute removed or exchanged, so that
$\delta_{g(x,\bar y)} = 0$. The \emph{tilt} of the policy toward the attribute at $(x,y)$ is
the extra score it gives to the version that carries the attribute, beyond what the
reference gives it:
\[
    a(x,y) \coloneqq r(x,y) - r(x,\bar y)
    = \beta\left[\log \frac{\pi(y \mid x)}{\pi(\bar y \mid x)}
      - \log \frac{\piref(y \mid x)}{\piref(\bar y \mid x)}\right].
\]
Every score then splits into the score of the attribute-free version and the tilt,
\[
    r(x,y) = q(x,y) + a(x,y)\,\delta_{g(x,y)},
\]
where $q(x,y)$ is $r(x,\bar y)$ when $y$ carries the attribute and $r(x,y)$ otherwise. This
is a change of variables, not a model: $q$ is the part of the score that concerns the answer
itself, and $a$ is the part that concerns the attribute. The tilt is a function of $(x,y)$,
which the policy can represent (Proposition~\ref{prop:absorb}); the case to keep in mind is
the one in which it takes a single value $a$ for every response, since that is the part of
the tilt on which the degeneracy acts.

The two kinds of comparison in the data now read differently. On a pair whose two responses
differ only in the attribute, $q$ cancels and the logit of annotator $k$ is
\[
    a + \theta_k,
\]
the policy's tilt plus the annotator's bias and nothing else. On a pair whose two responses
carry the same attribute value, the tilt cancels and the logit is $q(x,y_w) - q(x,y_l)$,
quality alone. Mixed pairs combine the two.

\subsection{What the labels determine}
\label{app:anchor_identified}

The transformation of Corollary~\ref{prop:degeneracy}, $r \mapsto r + c\,\delta_g$ and
$\theta_k \mapsto \theta_k - c$ for every $k$, adds $c$ to the tilt of every response and
subtracts $c$ from every bias, and every logit is unchanged. In the coordinates above,
Theorem~\ref{thm:identif} says three things.
\begin{enumerate}[nosep,leftmargin=*]
\item The quality part $q$ is determined up to the shifts that standard DPO cannot see
either, a constant within each compared pair.
\item For every annotator the sum $a + \theta_k$ is determined. Hence every difference
$\theta_k - \theta_{k'}$ is determined, and in the parameterisation
$\theta_k = \bar\theta + \varepsilon_k$ every deviation $\varepsilon_k$ is determined.
\item The split of $a + \bar\theta$ into $a$ and $\bar\theta$ is not determined.
\end{enumerate}
The common offset is therefore not a parameter but a direction: the line
$a + \bar\theta = \text{const}$ in the $(a, \bar\theta)$ plane, along which the likelihood
is flat. Fixing the model means choosing a point on this line, and since every point fits
every judgment equally well, the choice is an input, not an estimate. The statistical
content is the one stated after Corollary~\ref{prop:degeneracy}: a bias shared by every
annotator cannot be told apart from the attribute genuinely carrying reward. Declaring the
attribute states that the shared preference for it is bias; the anchor states from which
zero that bias is counted.

\subsection{Reference anchoring}
\label{app:anchor_reference}

Every run in this paper starts from the reference, $\pi = \piref$, so at initialisation
$r \equiv 0$, $a = 0$ and $\theta = 0$. Training moves $a + \bar\theta$ to the value the
labels demand. Which of the two moves is decided by the dynamics, not by the objective: the
scalar has a convex path to its conditional optimum (Theorem~\ref{thm:identif}) and its
own learning rate, whereas the policy moves through a non-convex landscape under the KL
anchoring, so the level goes into $\bar\theta$ and $a$ stays near zero
(Section~\ref{sec:mechanism}; Proposition~\ref{prop:rate} explains why the naive
per-annotator model gets there $m$ times more slowly). The learned $\bar\theta$ is therefore
the annotators' shared tilt toward the attribute measured relative to the reference, and the
debiased policy's attribute rate is the reference's rate.

Three consequences follow. First, removal is measured against the reference, which is why
the tables report the share of DPO's increase over the reference that is removed, and why
the 0.5B reference of Signed-UltraFeedback is balanced on both attributes by construction
(Appendix~\ref{app:names}): the bias then arrives through the preference stage alone.
Second, a reference that is itself tilted stays tilted by its own amount: at 8B the
reference signs with a woman-coded name at $0.465$, and \method{} returns to $0.49$, not to
$0.50$ (Table~\ref{tab:names_v3}). Third, the estimates of the annotators do not depend on
the reference's tilt: $r \equiv 0$ at the start whatever the reference prefers, so
$\bar\theta$ measures the tilt of the labels beyond zero implicit reward, and the deviations
$\varepsilon_k$ are the same at every point of the line by item (ii) above. The audit of the
annotators is anchor-independent; only the policy's rate depends on the anchor.

\subsection{Target anchoring}
\label{app:anchor_target}

The same free direction can be fixed by a requirement on the outputs instead of by the
reference. Suppose the attribute has a fair rate $t$, one half for the name in a signature,
and write $\operatorname{logit} p = \log\frac{p}{1-p}$. Let $p$ be the reference's rate in the
pairwise sense, $\log \piref(y \mid x)/\piref(\bar y \mid x) = \operatorname{logit} p$. A
policy with tilt $a$ has log-odds $\operatorname{logit} p + a/\beta$, so the tilt that reaches
the target is
\begin{equation}
    c(t) = \beta\,\bigl[\operatorname{logit} t - \operatorname{logit} p\bigr].
    \label{eq:anchor_tilt}
\end{equation}
When the reference's log-odds vary across prompts, the rate of the tilted policy is still
continuous and strictly increasing in $c$, so the target is reached at exactly one value,
found by a one-dimensional search; \eqref{eq:anchor_tilt} is that value for constant
log-odds.

\begin{corollary}[Target anchoring]
\label{cor:anchor}
Let $(r, \theta)$ be any parameter vector and $c \in \R$. The parameter vector
$(r + c\,\delta_g,\; \theta - c\mathbf{1})$ assigns the same probability to every judgment,
has the same deviations $\varepsilon_k$, has shared mean $\bar\theta - c$, and corresponds to
the policy $\pi_c(y \mid x) \propto \pi(y \mid x)\exp\bigl(c\,\delta_{g(x,y)}/\beta\bigr)$,
whose log-odds between a response and its counterfactual exceed those of $\pi$ by $c/\beta$.
\end{corollary}
\begin{proof}
The first three claims are Corollary~\ref{prop:degeneracy}. The policy follows from
$\pi \propto \piref\exp(r/\beta)$ (Proposition~\ref{prop:compat}) with $r + c\,\delta_g$ in
place of $r$, and the log-odds shift is $c/\beta$ because $\delta_g$ differs by one between
$y$ and $\bar y$.
\end{proof}

Target anchoring therefore changes exactly one thing. The fit to every judgment is the same,
so nothing is paid in likelihood. The deviations are the same, so the audit of the annotators
is the same. The shared mean shifts by $-c(t)$, which means that the reference's own tilt is
now counted as bias rather than as quality. What changes is the policy, which is moved along
the attribute by $c(t)/\beta$ in log-odds, and with it the KL to the reference. The
requirement spends the one degree of freedom the data do not use, and nothing else moves.

\paragraph{How the anchor is applied.}
Three procedures produce the tilted policy. At sampling time, in closed form: the responses
that carry the attribute are reweighted by $\exp(c/\beta)$ per prompt, which for a signature
is a shift on the choice of the name and for an attribute of the whole text is a reweighting
or a rejection step. In training, with a fixed offset: on pairs that differ only in the
attribute, each taken in both orders, the loss of the two orders is
$-\log\sigma(u + b) - \log\sigma(-u - b)$ with $u$ the policy's margin, minimised at $u = -b$,
so a fixed offset $b = -c(t)$ in the logit drives the policy to carry the tilt $c(t)$ itself.
This removes a tilt a policy already has, from its own generations and with no annotator
labels, and is the natural cure for a model trained without the bias term;
Section~\ref{app:anchor_experiment} tests it. Through the reference: a reference balanced on the attribute makes the two anchors
coincide, which is the construction used at 0.5B.

\paragraph{Limits.}
Target anchoring needs an attribute with a fair value. A name has one; length and formatting
do not, and for them the reference remains the only sensible anchor, which is why
Table~\ref{tab:multipref} reads the held-out gap rather than a target rate. The target is a
normative choice, as the attribute partition is, but it is one that can be audited by
measuring the rate, whereas the reference's rate is an accident of pretraining. Moving the
policy changes which responses it produces, so judged quality has to be checked after the
tilt. Finally, the level can also be pinned from data rather than from a target: a small set
of judgments by annotators known to be unbiased breaks the flat direction, since for them
$\theta_k = 0$ is known and $a$ is then identified. We leave this to future work
(Section~\ref{sec:conclusion}).

\paragraph{Worked example.}
Take a response signed \emph{Emily Miller} and its counterfactual signed \emph{Greg Miller},
the same text otherwise. The policy's tilt $a$ is $\beta$ times the difference between the
policy's and the reference's log-odds for the first version over the second. When annotator
$k$ judges this pair, the logit is $a + \theta_k$: the labels of the swap pairs fix, for every
annotator, the sum of the policy's tilt and the annotator's bias, but not how it splits. At
initialisation $a = 0$, and after training under reference anchoring $a$ is still near zero
and $\bar\theta$ holds the shared level. For the numbers, take $\beta = 0.1$. The 8B reference
signs with a woman-coded name at $0.465$, so $\operatorname{logit} p = -0.14$ and reaching
$t = 0.5$ needs $c = 0.1 \times (0 + 0.14) = 0.014$, a shift of $0.14$ in log-odds: for a
nearly balanced reference the two anchors almost coincide. The 8B DPO policy trained on the
same labels signs with a woman-coded name at $0.967$, so $\operatorname{logit} p = 3.38$ and
curing it needs $c = -0.34$, a shift of $-3.4$ in log-odds, which the training-time procedure
obtains with the offset $b = 0.34$ on the gender coordinate.

\subsection{Experiment: a fixed offset cures the DPO policy}
\label{app:anchor_experiment}

\paragraph{Setting.}
The source is the DPO policy of Signed-UltraFeedback at 0.5B (Table~\ref{tab:names_v3}), one
per seed, which signs with a woman-coded name with probability $0.96$ and with a black-coded
name with probability $0.99$. The training pairs come from the policy itself: $3{,}000$ fresh
UltraFeedback prompts (the ones after the $8{,}000$ of the corpus) with the signing instruction,
answered by the source; the answers signed with a pool name are kept ($2{,}659$ at seed 42);
each signature is replaced by a random cell and then by the cell that differs from it in one
attribute, with the same surname, which gives one pair per attribute per answer, taken in both
orders: $10{,}636$ rows, of which $9{,}576$ are used for training and $1{,}060$ held out. No annotator and no label enter this stage. Since every pair is present
in both orders, the loss on a pair is $-\log\sigma(u+b)-\log\sigma(-u-b)$ with $u$ the policy's
margin, minimised at $u=-b$, so the offset $b$ in the logit (the term $b\,(g_w-g_l)$ of
Section~\ref{app:anchor_target}) is the only thing the stage teaches. Training is DPO with the
source as reference and the hyperparameters of Table~\ref{tab:hyperparams} ($\beta = 0.1$,
learning rate $5\times10^{-7}$, $1{,}000$ steps of $32$ pairs). The offset for a target rate $t$
is $b = -\beta s^\ast(t)$, where $s^\ast(t)$ is the log-odds shift that brings the source's
signature probabilities to $t$, found by bisection on the $300$ evaluation prompts: for parity
at seed 42, $s^\ast = 3.39$ on the gender coordinate and $5.06$ on the race coordinate.

\paragraph{Calibration.}
The exact tilt alone stops short: at seed 42 it reaches $0.736$ and $0.753$ instead of $0.50$,
two thirds of the intended shift in log-odds. Three offsets at seed 42 ($3.39/5.06$,
$4.42/6.18$ and $5.13/6.92$ in log-odds, gender/race) give achieved shifts of $2.24/3.68$,
$3.13/4.76$ and $3.79/5.51$, which lie on a line with residuals below $0.015$: achieved $=
0.888\,\text{intended} - 0.777$ on gender and $0.980\,\text{intended} - 1.285$ on race. The
policy therefore follows the offset one for one after a fixed deficit of $0.8$ and $1.3$
log-odds, which the $1{,}000$ steps leave. Every other run below inverts this line for the exact
tilt of its target; the line is fitted at seed 42 and applied unchanged to the other two
source seeds and to the targets $0.3$ and $0.7$.

\begin{table}[t]
\centering
\footnotesize
\setlength{\tabcolsep}{2pt}
\caption{A fixed offset applied to the DPO policy of Signed-UltraFeedback at 0.5B, trained on
the policy's own generations with no annotator. Rates are the probability readout of
Table~\ref{tab:names_v3}; RM$^-$ is the judge with the signature removed; KL is per token, to the
source (the DPO policy) and to the SFT reference; ${<}0.01$ marks an estimate below the
noise of the Monte Carlo estimate, which can come out negative; the held-out gap is defined in
Section~\ref{sec:setup}; $t$ is the target rate; \emph{Signed} is the share of answers that carry a
signature. Seed 42 unless stated; the parity row over three seeds gives the mean
and the 95\% interval over the source seeds.}
\label{tab:offset}
\begin{tabular}{@{}llccccccc@{}}
\toprule
\textbf{Policy} & $t$ & $p(\text{woman})$ & $p(\text{black})$ & \textbf{Signed} & \textbf{Tokens} & \textbf{RM}$^-$ & \textbf{KL} src / SFT & \textbf{Gap} \\
\midrule
DPO (source) & -- & 0.963 & 0.992 & 0.84 & 213 & $-2.75$ & -- / 0.037 & $+0.095$ \\
+ offset, exact tilt & 0.5 & 0.736 & 0.753 & 0.86 & 210 & $-2.79$ & ${<}0.01$ / 0.028 & $+0.029$ \\
+ offset, calibrated & 0.5 & 0.533 & 0.509 & 0.86 & 214 & $-2.71$ & ${<}0.01$ / 0.023 & $-0.046$ \\
+ offset, calibrated, 3 seeds & 0.5 & $0.50 \pm 0.07$ & $0.47 \pm 0.08$ & 0.86 & 213 & $-2.72$ & 0.003 / 0.022 & $-0.053$ \\
+ offset, calibrated & 0.3 & 0.262 & 0.233 & 0.84 & 217 & $-2.77$ & 0.034 / 0.024 & $-0.144$ \\
+ offset, calibrated & 0.7 & 0.681 & 0.659 & 0.85 & 206 & $-2.90$ & ${<}0.01$ / 0.027 & $+0.019$ \\
\midrule
\method{} pooled, Table~\ref{tab:names_v3} & -- & 0.545 & 0.616 & 0.86 & 213 & $-2.53$ & -- / 0.025 & $+0.005$ \\
\bottomrule
\end{tabular}
\end{table}

\paragraph{Results.}
Table~\ref{tab:offset} reports the runs. With the calibrated offset the DPO policy signs with a
woman-coded name at $0.50 \pm 0.07$ and with a black-coded name at $0.47 \pm 0.08$ over the three
source seeds, from $0.96$ and $0.99$; the targets $0.3$ and $0.7$ are reached within $0.07$ and
$0.04$, and the rate is monotone in the offset. Nothing else moves: the signing rate, the length
and the judge score with the signature removed stay within their seed spread, the KL to the
source is at the noise floor of the estimate except at the largest offset ($0.034$), and the KL to
the SFT reference falls from DPO's $0.037$ to $0.022$ at parity, below the $0.025$ of \method{},
since the part of DPO's divergence that the offset removes is the name shift itself. The held-out
gap, which marks a policy that has internalised the bias, falls from DPO's $0.095$ to $0.03$
after the exact tilt and to $-0.05$ at parity, where the cross-group accuracy ($0.50$ to $0.53$)
no longer exceeds the same-group one; at the target $0.3$ it reaches $-0.14$, the policy now
siding against the annotators on the pairs that differ in the attribute. The one remaining deficit is the
constant one of the calibration paragraph: a property of the training budget, not of the
objective, which a longer run or the second calibration point removes. Two remarks. First, the
sampled signatures lag the probability readout at mild offsets and swing past it at strong ones
(at parity, $69\%$ of the sampled signatures at seed 42 still carry a black-woman name while the
pool probabilities are balanced; at the target $0.3$, $92\%$ carry a white-man name), the
concentration of the decoder on a few names described in Appendix~\ref{app:ablations}; the
probability readout is the quantity the offset controls. Second, the offset is the same quantity
as the sampling-time tilt of Section~\ref{app:anchor_target} up to the fixed deficit, so a policy
can be cured either by one training stage on its own generations or by reweighting at sampling
time, with the same target.

\section{Datasets}
\label{app:datasets}

Table~\ref{tab:datasets} lists the two corpora, their sizes and the modifications made to each; the subsections give the construction in full. The 8B runs use the same data as
the 0.5B runs.

\begin{table}[h]
\centering
\caption{The corpora. Judgments are training / held-out rows; the held-out rows are
$10\%$ of prompts, split by prompt id. On MultiPref the $300$ evaluation prompts come from the held-out side; on
Signed-UltraFeedback they come from UltraFeedback's test split. Cross-group is the share of judgments in which exactly one response
carries the attribute.}
\label{tab:datasets}
\footnotesize
\setlength{\tabcolsep}{3pt}
\begin{tabular}{@{}p{1.6cm}p{2.7cm}p{2.0cm}p{2.7cm}p{4.0cm}@{}}
\toprule
\textbf{Corpus} & \textbf{Size} & \textbf{Annotators} & \textbf{Attributes (cross-group)} & \textbf{Modifications} \\
\midrule
MultiPref \citep{multipref} & $10{,}461$ comparisons over $5{,}323$ prompts; $27{,}812$ / $3{,}035$ judgments & $227$ real evaluators, four per comparison & length, ratio $\geq 1.5$ ($58.7\%$); formatting, markdown present ($26.4\%$) & ties dropped ($26\%$ of rows); judgments kept disaggregated with the evaluator id; attributes computed from the response texts \\
\addlinespace
Signed-Ultra\-Feedback & $7{,}177$ UltraFeedback pairs (from $8{,}000$, responses under five words dropped) in three types: $2{,}887$ quality, $1{,}462$ swap, $2{,}828$ mixed; $25{,}840$ / $2{,}868$ judgments & $60$ with planted biases in three classes of $20$, $\theta_k \in \R^2$ & woman-coded and black-coded first name in the signature ($36.4\%$ each) & signature line \texttt{--- First Last} appended, names from an 89-name pool; swap pairs duplicate one response under two names; labels re-drawn; SFT answers re-signed at random so the reference starts balanced \\
\bottomrule
\end{tabular}
\end{table}

\subsection{The MultiPref corpus and the length attribute}
\label{app:multipref}

\texttt{allenai/multipref} contains $10{,}461$ comparisons, each judged by exactly four
evaluators (two crowd workers and two experts; $227$ distinct evaluators, no overlap
between the two pools), over $5{,}323$ distinct prompts. We keep every judgment as its
own training row, so that annotator identity is available to the per-annotator model,
and drop the $26\%$ of judgments marked as ties, leaving $30{,}847$ rows. Ten percent
of prompts are held out by prompt id ($27{,}812$ train, $3{,}035$ held-out rows); the
$300$ evaluation prompts are drawn from the held-out prompts. The reference policy is
fine-tuned on the majority-chosen response of each comparison and is shared by every
method and seed.

\paragraph{Length as a binary attribute.}
A response is assigned $\delta_g = 1$ when it is markedly longer than its partner,
defined as a word-count ratio of at least $1.5$; pairs closer in length than that are
same-group: both responses have $\delta_g = 0$. Because this attribute is defined on the pair, it is not a
function of $(x,y)$: Proposition~\ref{prop:compat} and Theorem~\ref{thm:identif} need only
a $\delta_g$ per judgment and are unaffected, whereas the representability of the shared
component inside the policy's implicit reward (Proposition~\ref{prop:absorb}) holds for
length only through a per-response proxy such as the token count. The threshold is part of the attribute's definition and
not a tuning choice: assigning $\delta_g = 1$ to the longer side of every pair makes
$99.5\%$ of MultiPref comparisons cross-group and leaves no same-group comparisons to
identify the reward. At ratio $1.5$, $58.7\%$ of judgments are cross-group, the longer
response wins $73\%$ of those, and the corpus-level offline bias estimate is
$\hat\theta = 0.99$ in log-odds. Ratios of $1.25$, $2$ and $3$ give cross-group
fractions of $73\%$, $40\%$ and $21\%$ with $\hat\theta$ of $0.91$, $1.04$ and $1.05$;
the estimate is stable across thresholds.

\paragraph{Formatting as a binary attribute.}
A response is assigned $\delta_g = 1$ when it contains markdown markup. The attribute is
declared on the same $27{,}812$ training and $3{,}035$ held-out judgments as length;
$26.4\%$ of judgments are cross-group and the offline estimate is $\hat\theta = 0.98$.
Table~\ref{tab:attributes} characterises both attributes from the preference data
before any training.

\begin{table}[h]
\centering
\caption{The declared attributes on MultiPref, characterised from the preference data
before any training. \emph{Cross-group} is the fraction of judgments in which exactly
one response carries the attribute, which are the only judgments that identify
$\theta$. $\hat\theta$ is the log-odds that the attribute-carrying side wins such a
judgment, restated as odds in the next column. The last column is DPO's measured
amplification over the reference at 0.5B, in units of its own per-seed standard
deviation.}
\label{tab:attributes}
\begin{tabular}{@{}lcccc@{}}
\toprule
\textbf{Attribute} & \textbf{Cross-group} & $\hat\theta$ & \textbf{Odds} & \textbf{DPO amplification} \\
\midrule
Length & 58.7\% & 0.99 & 2.7 : 1 longer & $+96$ tokens ($65\sigma$) \\
Formatting & 26.4\% & 0.98 & 2.7 : 1 markdown & $+0.034$ ($1.6\sigma$) \\
\bottomrule
\end{tabular}
\end{table}

\subsection{The \textsc{names} corpus (Signed-UltraFeedback)}
\label{app:names}

The \textsc{names} corpus is built rather than collected: the responses come from a public
preference corpus, a name signature is added to them, and annotators whose biases are
planted relabel every pair, so the bias of every annotator is known. Two attributes are planted at once, and the
annotators are organised in classes whose biases differ. The main text calls it
Signed-UltraFeedback (Section~\ref{sec:names_v3}). We stress that the annotators are
constructed to prefer a name marker. It does not measure prejudice, and we refer to the
attributes as \emph{gender-coded} and \emph{race-coded names} rather than as gender or race
bias.

\paragraph{Prompts and responses.}
We draw $8{,}000$ preference pairs from the \texttt{train\_prefs} split of
\texttt{ultrafeedback\_binarized} \citep{ultrafeedback} and discard pairs in which either
response has fewer than five words, which leaves $7{,}177$. Each pair carries a prompt $x$, two
responses and two GPT-4 quality scores $s_1, s_2 \in [1,10]$ from the original dataset.

\paragraph{Marker.}
The marker is a signature line \texttt{--- First Last} appended to a response. First names
come from four cells (white woman, white man, black woman, black man): the audit lists of
\citet{bertrand2004emily} together with the WEAT~3 and WEAT~5 sets of
\citet{caliskan2017semantics}, $23/23/18/25$ names per cell, $89$ in all. Surnames come from a
neutral pool and are shared by both sides of a pair, so within a pair only the first name, and
hence only gender or race, differs. A response carries a two-entry attribute vector
$\delta_g(y) = (\mathbf{1}[\text{woman-coded}], \mathbf{1}[\text{black-coded}])$, read from
the signature; an unsigned response is $(0,0)$. Token balance across cells is not achievable
with published names: under the Qwen tokenizer nearly every white-coded name is one token and
most black-coded names are two or three (mean $1.1$, $1.0$, $2.1$ and $2.0$ tokens per cell).
We keep the full pool and read the rate from probabilities over every name
(Appendix~\ref{app:abl_names}).

\paragraph{Pair types.}
Three kinds of pair share the corpus, in proportion $40/20/40$. \emph{Quality} pairs are
the two real responses, both unsigned or both signed with the same name, and identify
quality alone. \emph{Swap} pairs are one response duplicated and signed with two names
that differ on exactly one attribute, with true quality margin $q=0$; they identify the
bias with quality held equal by construction. \emph{Mixed} pairs are the two real
responses signed from different cells, where quality and bias compete. The realised
counts are $2{,}887$, $1{,}462$ and $2{,}828$.

\paragraph{Annotators and labels.}
Sixty annotators sit in three classes of twenty. Annotator $k$ in class $c$ has
$\theta_k = \mu_c + \varepsilon_k$ with $\varepsilon_k \sim \mathcal{N}(0, 0.8^2 I)$. The
class means are $\mu_A = (1.2, 2.5)$, $\mu_B = (0.8, -0.5)$ and $\mu_C = (1.0, 1.0)$. Both
attributes have a population mean of $1.0$ and differ only in how much the classes disagree:
the class standard deviation is $0.16$ on the gender-coded attribute and $1.22$ on the
race-coded one, where class $B$ opposes the others. With twenty annotators per class the
realised class means are $0.98$, $1.12$ and $0.89$ on the gender-coded attribute and $2.62$,
$-0.34$ and $1.11$ on the race-coded one (pooled $1.00$ and $1.13$). Each pair receives four judgments. For each judgment
an annotator $k$ is drawn uniformly and the originally chosen response $y_1$ is preferred to
$y_2$ with probability
\[
    \sigma\!\bigl(q + \theta_k^\top(\delta_g(y_1) - \delta_g(y_2))\bigr),
    \qquad q = \kappa\,(s_1 - s_2),\ \kappa = 0.5,
\]
otherwise the label is flipped. Pairs whose responses carry the same attribute values are
thus labelled from quality alone, while the others are tilted by the annotator's bias. This
is the likelihood of \eqref{eq:biased_bt}, so the corpus tests whether the method recovers
what it is specified to recover; MultiPref is where the model may be misspecified. Ten
percent of prompts are held out by prompt id, which gives $25{,}840$ training and $2{,}868$
held-out judgments over the same annotators;
the reference is fine-tuned on the majority-chosen response of each pair, each re-signed with a
name from a random cell before fine-tuning, so the reference starts balanced on both attributes
and all of the bias arrives through the preference stage.

\paragraph{Gate.}
The following statistics are read from the labels before any training. $36.4\%$ of judgments
are cross-group on each attribute, $59.8\%$ on at least one. On swap pairs the offline log-odds
that the attribute-carrying side wins is $0.95$ for the gender-coded attribute and $0.70$ for
the race-coded one; by class, $1.01$, $0.97$ and $0.86$, and $2.30$, $-0.37$ and $0.88$. The
pooled race-coded estimate is below the gender-coded one although its planted mean is higher,
because class $B$ votes against the other two. On all cross-group rows both estimates are
$0.71$ and $0.70$, below the swap-pair values, because the quality margin adds noise wherever
the two texts differ.

\paragraph{Evaluation prompts.}
Three hundred held-out questions from \texttt{test\_prefs}, each followed by ``Sign your
answer with a name of your choice'', give the attribute rate:
the probability the policy assigns to the names of each cell at the signature position
(Appendix~\ref{app:readouts}). One hundred templated prompts present
two candidates with identical qualifications and names from different cells, in random
order, and ask for a recommendation; they are read as the share of recommendations going to
each cell and are not used in the tables.

\paragraph{Configuration.}
In summary: $20$ annotators per class, a within-class standard deviation of $0.8$, the class
means above, quality, swap and mixed pairs in proportion $40/20/40$ with half of the quality
pairs unsigned, four judgments per pair, a quality scale $\kappa = 0.5$, and data seed $0$.

\section{Experimental details}
\label{app:training}

\subsection{Models and training}

\texttt{Qwen2.5-0.5B-Instruct} is trained with full fine-tuning: fp32 master weights
under bf16 autocast, because bf16 master weights lose updates at a DPO learning rate
of $5\times10^{-7}$. \texttt{Llama-3.1-8B-Instruct} and \texttt{Mistral-7B-Instruct-v0.3}
are trained with LoRA on all linear layers (rank 32, $\alpha = 64$, fp32 adapters on a bf16 base) at learning rate
$5\times10^{-6}$, ten times the full fine-tuning rate; the rate was fixed by the
calibration of Appendix~\ref{app:abl_lora}, in which LoRA at this rate reproduces the
amplification and the removal of full fine-tuning to within seed noise, while a larger rate lets the policy move faster than the bias term; Mistral uses the same recipe unchanged. All models use $\beta = 0.1$, 1000 steps at
effective batch 32, and a separate Adam at learning rate $10^{-2}$ for the bias
parameters, since the scalars otherwise receive too little gradient to move. One SFT reference per model and corpus (a LoRA reference for the two larger models), trained for one
epoch on the majority-chosen response of each comparison, is shared by every arm and
seed; without it the implicit rewards would not be comparable across arms. Reference
log-probabilities are precomputed once per corpus and cached, and per-position
log-probabilities are computed in position chunks so that the full vocabulary logits
never materialise at once. Every arm sees the identical disaggregated judgments through
the same implementation; the arms differ only in their loss. All runs use a single shared
NVIDIA A100 80GB. A 0.5B preference run peaks at 11 to 19\,GB and takes 70 to 150
minutes depending on contention; an 8B LoRA run peaks at 21.5\,GB and takes 7 to 9
hours, a Mistral run at 17.7\,GB and 8 to 12 hours. Table~\ref{tab:hyperparams} lists the hyperparameters.

\begin{table}[h]
\centering
\caption{Hyperparameters. The reference checkpoint is shared across all methods and seeds.}
\label{tab:hyperparams}
\begin{tabular}{@{}lp{0.72\linewidth}@{}}
\toprule
\textbf{Stage} & \textbf{Setting} \\
\midrule
SFT (reference)     & 1 epoch on majority-chosen responses, effective batch 32 \\
SFT lr              & $1\times 10^{-5}$ (0.5B, full); $1\times 10^{-4}$ (8B and 7B, LoRA) \\
Preference stage    & $\beta = 0.1$, 1000 steps, effective batch 32 ($2 \times 16$) \\
Preference lr       & $5\times 10^{-7}$ (0.5B, full); $5\times 10^{-6}$ (8B and 7B, LoRA) \\
LoRA (8B and 7B)    & rank 32, $\alpha = 64$, all linear layers, fp32 adapters on a bf16 base \\
LoRA reference      & a LoRA SFT of the same shape, shared by every arm and seed \\
Sequence lengths    & at most 1280 tokens for prompt and response together, of which at most 384 for the prompt \\
Schedule            & cosine with 10\% warmup, gradient clipping 1.0 \\
Bias parameters     & separate Adam at lr $10^{-2}$, no penalty, initialised at 0 unless warm-started \\
Generation          & 300 held-out prompts, $\leq 512$ new tokens, $T = 0.7$, top-$p$ $0.9$ \\
Judge               & \texttt{Skywork-Reward-V2-Qwen3-1.7B}, truncated to 2048 tokens \\
Seeds               & 42, 123, 456 \\
\bottomrule
\end{tabular}
\end{table}

\subsection{Arms}
\label{app:arms}

\emph{Reference} is the shared SFT policy and the zero point of every comparison.
\emph{DPO} is standard DPO, that is \method{} with every $\theta_k$ frozen at zero. The
\method{} arms differ only in how the bias term of \eqref{eq:barp_dpo_loss} is indexed
(Section~\ref{sec:variants}). \emph{Pooled} learns one scalar $\theta$ per attribute
shared by every comparison. $\bar\theta+\varepsilon_k$ learns a shared mean per
attribute plus one deviation per annotator, with the mean updated on every cross-group
comparison. \emph{$\theta_k$ only} is the original BARP model, one free parameter per
annotator and nothing shared. \emph{Shuffled ids} is the $\theta_k$-only model in which
every judgment is assigned an annotator index drawn uniformly at random, independently
of who cast it, so that every $\theta_k$ collects votes from every annotator and the ids
carry no information; a mere relabelling of the ids would change nothing and is not what
this control does. $\bar\theta+\delta_c+\varepsilon_k$ inserts a per-class offset
between the mean and the individual deviation and receives the class labels. \emph{Warm} arms
initialise $\bar\theta$ (or the pooled $\theta$) at the corpus-level offline estimate,
the log-odds that the $\delta_g = 1$ side wins a cross-group comparison, instead of at
zero. \emph{Gender declared only} is the pooled arm with the race-coded attribute left
undeclared. The main text carries pooled and $\bar\theta+\varepsilon_k$; the others are
in Appendix~\ref{app:ablations}.

\subsection{Baselines}
\label{app:baselines}

R-DPO \citep{park2024disentangling} adds a length penalty
$\alpha\,(|y_w| - |y_l|)$ in tokens to the DPO logit; we use $\alpha = 0.005$, since
the published $\alpha = 0.02$ collapses the policy on MultiPref
(Appendix~\ref{app:abl_multipref}). SamPO \citep{lu2024sampo} down-samples the token
log-ratios of the longer response so that both responses contribute equally many terms
to the implicit reward. Neither reads the declared attribute, so their length
checkpoints are scored unchanged under formatting. Group-DRO DPO adapts
\citet{sagawa2020groupdro} to the DPO loss: each group carries a loss weight, updated by
exponentiated gradient with step $\eta = 0.5$ toward the group with the highest current
loss; groups are the annotator classes on Signed-UltraFeedback and the annotators on
MultiPref. Crowd-PrefRL \citep{chhan2024crowd}, adapted in the same way, moves the
weight away from the highest-loss annotator ($\eta = 0.5$). EM-DPO with MinMax-DPO
\citep{chidambaram2024emdpo} trains $K = 3$ type policies, the planted number of
classes: each annotator holds a probability of belonging to each type; the M-step runs
DPO for each type on judgments weighted by those probabilities, and the E-step updates
the probabilities from how well each type policy explains the annotator's judgments,
scored on votes the round's policies did not train on (each annotator's votes are split
into two halves that swap every round; four rounds of 200 steps, after which the three
policies train on all votes for the full 1000 steps). MinMax-DPO mixes the type policies
with the weights that minimise the largest regret over the types; the mixture answers
each prompt with type policy $k$ with probability $w_k$, so its readouts are the
weighted means of the type policies' readouts, and its vote prediction uses each
annotator's own type probabilities. No public implementation matched our data and single-GPU setting, so we reimplemented the method; departures from the original are documented in the released code.

\subsection{Readouts}
\label{app:readouts}

\paragraph{Generation.}
Each policy generates on 300 held-out prompts with temperature $0.7$, top-$p$ $0.9$ and
at most 512 new tokens.

\paragraph{Attribute rate on Signed-UltraFeedback.}
The rate is read from probabilities, not from samples. With the reference's answer body
fixed, the probability the model assigns to each of the 89 first names at the signature
position is renormalised over the pool and summed over the woman-coded and over the
black-coded names, giving $p(\text{woman})$ and $p(\text{black})$; the average over the
300 sign prompts is the rate. Sampled signature rates saturate and depend on the
decoder: the reference, balanced by probability at $0.50$ and $0.52$, signs $79\%$ of its
answers with a black-coded name at temperature $0.7$ with top-$p$ $0.9$
(Appendix~\ref{app:abl_names}). Using each arm's own body instead of the reference's
changes no number by more than $0.005$. Computing the rate requires one forward pass per prompt and name batch.

\paragraph{Attribute rate on MultiPref.}
Under length the rate is the mean token count of the generations. Under formatting it
is the markdown rate reweighted onto the reference generations' token-count distribution
(five quantile bins), so that a method which merely shortens its answers does not appear to remove bias.

\paragraph{Bias removed.}
DPO's amplification of the attribute is the gap between DPO's rate and the reference's,
and an arm's removal is the share of that gap it closes,
$(\text{rate}_{\mathrm{DPO}} - \text{rate}_{\mathrm{arm}}) /
(\text{rate}_{\mathrm{DPO}} - \text{rate}_{\mathrm{ref}})$, computed per seed against DPO
of the same seed and then averaged, so $0\%$ is DPO and $100\%$ is the reference.

\paragraph{Held-out accuracy and the gap.}
On the held-out judgments the policy's implicit reward margin $u$ of
\eqref{eq:shorthand} predicts the label by its sign. Accuracy is reported on same-group
pairs (the two responses agree on the attribute) and cross-group pairs (they differ),
and the gap is cross-group minus same-group. A policy that has internalised the
annotators' bias predicts their labels better exactly where the attribute differs, so
debiasing shows as the gap falling to zero while same-group accuracy is unchanged.

\paragraph{Judge.}
Quality is scored by \texttt{Skywork-Reward-V2-Qwen3-1.7B} \citep{skywork} on the
generations. Where an attribute-invariant transform exists the attribute is stripped
from every arm's generations equally before scoring (the signature on
Signed-UltraFeedback, the markup under formatting), written RM$^-$; under length none
exists and the raw score is reported. The judge's scale differs between the three models'
outputs, so scores are compared within a model only.

\paragraph{Distance from the reference.}
For every trained policy we report a Monte Carlo estimate of
$\mathrm{KL}(\pi_\theta \,\|\, \pi_{\mathrm{ref}})$ on the policy's own generations: the
mean over generated tokens of $\log \pi_\theta(y \mid x) - \log \pi_{\mathrm{ref}}(y \mid x)$
with $y \sim \pi_\theta$. It distinguishes a method that removes a bias by changing the policy from one that removes it by leaving the policy near the reference.

\paragraph{Vote prediction with the arm's own bias parameters.}
On the held-out judgments whose two responses differ on the attribute, the arm predicts
the voter's label by the sign of $u + \theta_k^\top(\delta_g(y_w) - \delta_g(y_l))$
with its own learned $\theta_k$ (pooled uses its single $\theta$; arms without a bias
term use $u$ alone; EM-DPO uses each annotator's type probabilities). The ceiling is
the same prediction with the planted $\theta_k$, averaged over arms and seeds.

\section{Ablations and controls}
\label{app:ablations}

Everything in this appendix except the last subsection was run on \texttt{Qwen2.5-0.5B-Instruct} with the training configuration of Section~\ref{sec:setup}, at three seeds unless a row says otherwise. It covers the variants of the bias model that
Section~\ref{sec:variants} defines and the main text does not table, why the names tables
read probabilities rather than samples, the
remaining baselines, the LoRA calibration on which the 8B training configuration rests, and a replication on \texttt{Mistral-7B-Instruct-v0.3}. Column definitions follow the main text; intervals are one standard deviation over seeds, except in the Mistral table, which uses the $95\%$ intervals of the main text.

\subsection{Variants of the bias model on Signed-UltraFeedback}
\label{app:abl_variants}

Table~\ref{tab:abl_names_v3} adds to Table~\ref{tab:names_v3} the arms not reported in the main text, defined in Appendix~\ref{app:arms}. The two
length baselines, which the rule of Section~\ref{sec:setup} excludes from this corpus, are included for completeness.

\begin{table}[h]
\centering
\caption{Signed-UltraFeedback at 0.5B: the arms not reported in the main text, under the three arms of Table~\ref{tab:names_v3} for reference. Rows without a method name are
\method{} variants. Rate columns as in Table~\ref{tab:names_v3}, with the share of DPO's
amplification removed in parentheses; votes: vote prediction with the arm's own
$\theta$ on the race-coded / gender-coded attribute overall (Table~\ref{tab:votes}).
The judge score is inside DPO's $95\%$ interval for every arm and is omitted. Mean $\pm$ sd
over three seeds; the gender-only control ran at seed 42.}
\label{tab:abl_names_v3}
\small
\setlength{\tabcolsep}{3pt}
\begin{tabular}{@{}lcccc@{}}
\toprule
\textbf{Method} & $p(\text{woman})$ (removed) & $p(\text{black})$ (removed) & \textbf{KL} & \textbf{Votes r / g} \\
\midrule
DPO & $0.964 \pm 0.002$ & $0.991 \pm 0.001$ & $0.038 \pm {<}0.001$ & 0.641 / 0.673 \\
Pooled & $0.549 \pm 0.003$ (89\%) & $0.608 \pm 0.007$ (81\%) & $0.026 \pm {<}0.001$ & 0.649 / 0.694 \\
$\bar\theta+\varepsilon_k$ & $0.542 \pm 0.004$ (91\%) & $0.595 \pm 0.007$ (83\%) & $0.026 \pm {<}0.001$ & 0.752 / 0.745 \\
\midrule
Class, $\bar\theta+\delta_c+\varepsilon_k$ & $0.526 \pm 0.004$ (94\%) & $0.582 \pm 0.006$ (86\%) & $0.026 \pm {<}0.001$ & 0.749 / 0.740 \\
Shuffled ids & $0.849 \pm 0.009$ (25\%) & $0.925 \pm 0.008$ (14\%) & $0.034 \pm {<}0.001$ & 0.645 / 0.694 \\
Pooled, gender declared only & $0.836$ (28\%) & $0.986$ (1\%) & 0.034 & 0.636 / 0.689 \\
\midrule
R-DPO ($\alpha = 0.005$) & $0.967 \pm 0.002$ ($-1\%$) & $0.992 \pm 0.001$ (0\%) & $0.054 \pm 0.002$ & 0.634 / 0.659 \\
SamPO & $0.940 \pm 0.001$ (5\%) & $0.981 \pm 0.002$ (2\%) & $0.041 \pm 0.001$ & 0.625 / 0.669 \\
\bottomrule
\end{tabular}
\end{table}

The class arm is the best arm on both attributes at every seed, ahead of pooled by
$0.02$ on each, the same margin on the attribute the classes agree on and on the one
they do not, so the gain is not specific to disagreement. Its race-coded class offsets,
centred like the planted ones, come out at $+0.98$, $-1.00$ and $+0.03$ (three-seed means)
against planted offsets of $+1.49$, $-1.47$ and $-0.02$ around the population mean: the order
is recovered, the magnitudes are attenuated by about a third, and its vote prediction is at the ceiling, as for
$\bar\theta+\varepsilon_k$. When the classes are known it is a reasonable choice; it improves the policy by $0.02$ and adds nothing to what the deviations already reveal, which is why it is reported here rather than in the main text.
Shuffling the ids removes $25\%$ and $14\%$. Its sixty parameters converge to noisy copies of one value near $0.4$ (learned means $0.45$ and $0.39$, correlation with the planted
$\theta_k$ of $0.12$ and $-0.16$, means over three seeds), its vote prediction is at pooled's level, and its
KL is close to DPO's: with the identities destroyed, the parameters neither describe the annotators nor reach the mean within the training budget. Declaring only the gender-coded attribute leaves
the race-coded one at DPO's level ($0.986$) and removes only $28\%$ of the declared one.
The undeclared bias is absorbed by the policy, which concentrates its signatures in the
black-woman cell, and the declared attribute is carried along with it. Both attributes must be
declared for either to be removed, which is the substitution pattern of
Section~\ref{sec:formatting} inside one corpus. The two length methods remove at most
$5\%$ of either attribute, and R-DPO's length penalty raises the KL to the reference to
$0.054$ without changing the signature: a method built for length has no effect on a name signature.

\subsection{MultiPref: variants, warm starts and the remaining baselines}
\label{app:abl_multipref}

Table~\ref{tab:abl_multipref} is the length experiment of Table~\ref{tab:multipref} with
every arm we ran at 0.5B: the free-$\theta_k$ arm of the original BARP model, shuffled
ids, the two \method{} arms warm-started at the corpus-level offline estimate
$\hat\theta = 0.99$ instead of zero, and R-DPO at its published strength.
Group-DRO DPO is in Table~\ref{tab:multipref} at three seeds. Crowd-PrefRL
\citep{chhan2024crowd}, adapted to the DPO loss, moves weight away from the annotators with the
highest loss; at one seed it leaves the answers at $386$ tokens and the gap at $0.130$, against
DPO's $385$ and $0.126$.

\begin{table}[h]
\centering
\caption{MultiPref, length declared, every arm at 0.5B. Columns as in
Table~\ref{tab:multipref}; same / cross is held-out accuracy on same-group and
cross-group pairs, whose difference is the gap; \emph{mean} $\theta$ is the mean of the learned
bias parameters (the single $\theta$ for pooled). Core arms are mean $\pm$ sd over three seeds; other rows are single-seed
controls. These runs used an SFT reference trained before the one of Table~\ref{tab:multipref}, on
the same data; the reference and DPO rows here are from the same runs, so the removal shares are
computed within this table and differ from Table~\ref{tab:multipref} by a few points.}
\label{tab:abl_multipref}
\small
\begin{tabular}{@{}lccccc@{}}
\toprule
\textbf{Method} & \textbf{Tokens} & \textbf{Removed} & \textbf{RM} & \textbf{same / cross} & \textbf{mean} $\theta$ \\
\midrule
Reference (SFT) & 294.6 & --- & $-1.15$ & --- & --- \\
DPO & $385.4 \pm 3.5$ & --- & $-0.67$ & 0.521 / 0.640 & --- \\
\method{} (pooled) & $339.0 \pm 1.1$ & 51\% & $-0.69$ & 0.516 / 0.526 & 0.79 \\
\method{} (pooled, warm) & 329.5 & 62\% & $-0.67$ & 0.524 / 0.515 & 0.84 \\
\method{} ($\theta_k$ only) & $370.1 \pm 1.9$ & 17\% & $-0.67$ & 0.520 / 0.602 & 0.42 \\
\method{} (shuffled ids) & 372.1 & 16\% & $-0.77$ & 0.521 / 0.609 & 0.44 \\
\method{} ($\bar\theta+\varepsilon_k$) & $336.3 \pm 4.1$ & 54\% & $-0.68$ & 0.520 / 0.522 & 0.84 \\
\method{} ($\bar\theta+\varepsilon_k$, warm) & 335.4 & 56\% & $-0.73$ & 0.518 / 0.514 & 0.89 \\
R-DPO ($\alpha = 0.02$) & 137.8 & over-corr. & $-1.49$ & 0.467 / 0.287 & --- \\
R-DPO ($\alpha = 0.005$) & 309.7 & 84\% & $-0.90$ & 0.517 / 0.469 & --- \\
SamPO & 417.9 & $-34\%$ & $-0.72$ & 0.515 / 0.617 & --- \\
\bottomrule
\end{tabular}
\end{table}

Three observations follow. First, the free-$\theta_k$ arm removes $17\%$ with its mean at $0.42$, and shuffling the ids gives the same numbers ($16\%$, $0.44$): the $227$ free parameters do not remove the bias, and the identities play no part in that failure. Second, warm-starting the mean at $0.99$ gives the same result from the opposite direction: the pooled $\theta$ started at
$0.99$ settles at $0.84$ against $0.77$ to $0.81$ from zero, and the shared parameter
$\bar\theta$ at $0.78$ against $0.67$ to $0.72$, and the
resulting policies are indistinguishable. The policy never drives the parameter toward
zero, which is what an absorption account, in which the network out-competes the scalar
for the signal both can represent (Proposition~\ref{prop:absorb}), would predict. Third,
R-DPO at $\alpha = 0.02$ collapses the policy: answers fall to $138$ tokens, $157$ below
the reference, cross-group accuracy drops to $0.287$, so the policy prefers the shorter
answer two times in three, and the judge score is worse than the untrained reference's.
The heterogeneity baselines reweight the annotators strongly (weight sd $15$ and $6$ around a mean of $1$) and do not change length.

Under the formatting attribute the same variants repeat the pattern (three seeds each). The
free-$\theta_k$ arm and the shuffled arm learn means of $0.31$ and $0.32$ and leave the held-out
gap at $0.051$ and $0.056$, against $0.066$ for DPO; the pooled arm ($\theta = 0.71$) and the
shared-mean arm ($\bar\theta = 0.62$) close it to $0.004$ and $0.000$. The free mean is $0.50$
of $\bar\theta$, against $0.61$ under length. R-DPO at $\alpha = 0.02$, trained against the
earlier reference (Table~\ref{tab:abl_multipref}), raises the markdown rate to $0.700$, above
that build's DPO at $0.620$, with same-group accuracy at $0.389$: the optimisation pressure that
DPO placed on length has shifted to markdown in the shortened answers.

\subsection{Reading the signature from probabilities, not samples}
\label{app:abl_names}

Every names table reads the attribute rate from the probabilities the policy assigns to the
names in the pool (Appendix~\ref{app:readouts}), not from its sampled signatures. This
subsection shows why, on Signed-UltraFeedback. Table~\ref{tab:names_samples} gives both readouts
for the reference and four arms. The sampled rate is the share of signed answers to the $300$
sign prompts, generated at temperature $0.7$ and top-$p$ $0.9$, whose name is woman-coded or
black-coded.

\begin{table}[h]
\centering
\caption{Signed-UltraFeedback at 0.5B: the attribute rate read from sampled signatures (share of
signed answers) and from probabilities (Table~\ref{tab:names_v3}). Mean over three seeds.}
\label{tab:names_samples}
\small
\begin{tabular}{@{}lcccc@{}}
\toprule
\textbf{Method} & \textbf{Sampled, woman} & \textbf{Sampled, black} & $p(\text{woman})$ & $p(\text{black})$ \\
\midrule
Reference (SFT) & 0.61 & 0.79 & 0.498 & 0.517 \\
DPO & 1.00 & 1.00 & 0.964 & 0.991 \\
\method{} (pooled) & 0.74 & 0.91 & 0.549 & 0.608 \\
\method{} ($\bar\theta+\varepsilon_k$) & 0.69 & 0.91 & 0.542 & 0.595 \\
\method{} (shuffled ids) & 0.99 & 1.00 & 0.849 & 0.925 \\
\bottomrule
\end{tabular}
\end{table}

The two readouts disagree in three ways. First, the decoder distorts the rate before any
preference training. The reference is balanced by probability ($0.50$ and $0.52$), yet $79\%$
of its signed answers carry a black-coded name. The black-coded names share their first subword
(La-, Ta-, De-, Ja-), so at the first token of the name the probability is concentrated on a few
prefixes that top-$p$ sampling keeps, while the many distinct one-token white-coded names are
cut. Second, sampled rates saturate. DPO and the shuffled-id arm both sign almost every answer
with a black woman's name, so the samples cannot tell them apart, while by probability shuffled
ids sit at $0.85$ and $0.93$ against DPO's $0.96$ and $0.99$. Third, a shift in the sampled rate
is not evidence of annotator bias. Over all $300$ answers, with unsigned ones counted as not
black-coded, DPO raises the black-coded share from $0.70$ to $0.85$, by $0.15$. Relabelling the $89$-name pool at random $1{,}000$ times gives a null with standard
deviation $0.20$ and a $95\%$ range from $-0.37$ to $+0.35$ (one-sided $p = 0.22$, three DPO
seeds). Three names, \emph{Latisha}, \emph{Latonya} and \emph{Latoya}, carry $87\%$ of DPO's
signatures, and removing them from the black-coded set turns the shift into $-0.40$. The
probability readout reads every name in the pool and avoids all three problems.

\subsection{LoRA calibration at 0.5B}
\label{app:abl_lora}

Before the 8B runs, which use LoRA, we checked at 0.5B that a rank-32 adapter shows the same amplification and the same removal as full fine-tuning, and at which learning rate. The corpus was Signed-UltraFeedback with its balanced reference; the policy was the reference plus a LoRA adapter (rank 32, $\alpha = 64$, all linear layers, fp32 adapters on a bf16 base), trained for 1000 steps at seed 42 with DPO and with pooled \method{}, at learning rates $5\times10^{-6}$ and $2\times10^{-5}$, against $5\times10^{-7}$ for full fine-tuning.

\begin{table}[h]
\centering
\caption{LoRA calibration on Signed-UltraFeedback at 0.5B, seed 42. Rates, removed, RM$^-$
and KL as in Table~\ref{tab:names_v3}; $\theta$ is the learned pooled bias on the
gender-coded attribute. Removed is against the DPO run of the same training setting.}
\label{tab:lora}
\small
\setlength{\tabcolsep}{4pt}
\begin{tabular}{@{}lcccccc@{}}
\toprule
\textbf{Run} & $p(\text{woman})$ & $p(\text{black})$ & \textbf{Removed (g / r)} & \textbf{RM$^-$} & \textbf{KL} & $\theta$ \\
\midrule
Reference & 0.498 & 0.517 & --- & $-3.30$ & --- & --- \\
Full fine-tuning, DPO & 0.964 & 0.992 & --- & $-2.75$ & 0.037 & --- \\
Full fine-tuning, \method{} (pooled) & 0.545 & 0.616 & 90\% / 79\% & $-2.53$ & 0.025 & 0.68 \\
LoRA $5\times10^{-6}$, DPO & 0.952 & 0.988 & --- & $-2.71$ & 0.038 & --- \\
LoRA $5\times10^{-6}$, \method{} (pooled) & 0.529 & 0.592 & 93\% / 84\% & $-2.56$ & 0.027 & 0.68 \\
LoRA $2\times10^{-5}$, DPO & 0.996 & 0.998 & --- & $-2.46$ & 0.032 & --- \\
LoRA $2\times10^{-5}$, \method{} (pooled) & 0.637 & 0.734 & 72\% / 55\% & $-2.39$ & 0.025 & 0.67 \\
\bottomrule
\end{tabular}
\end{table}

At $5\times10^{-6}$, ten times the full fine-tuning rate, LoRA reproduces full fine-tuning on every readout: DPO amplifies to $0.95$ and $0.99$ against $0.96$ and $0.99$, pooled
removes $93\%$ and $84\%$ against $90\%$ and $79\%$, with the same learned $\theta$ ($0.68$)
and nearly the same KL ($0.027$ against $0.025$). A rank-32 adapter therefore suffices to show both the amplification and its removal, and this rate is used for the 8B runs. At $2\times10^{-5}$ DPO amplifies more ($1.00$) and pooled removes
clearly less ($72\%$ and $55\%$) although it learns nearly the same $\theta$ ($0.67$): the
policy moves faster than the bias term grows, so part of the bias is in the policy
before the scalar absorbs it. The rule is therefore a rate of about ten times the full fine-tuning rate and no higher, and the check to repeat on any new model is that DPO
under LoRA reaches the amplification of full fine-tuning.

\subsection{Robustness to the model family: Mistral-7B}
\label{app:abl_mistral}

The two models of Section~\ref{sec:experiments} differ in scale and in family. A third
model, \texttt{Mistral-7B-Instruct-v0.3} \citep{mistral7b}, differs from both in family
and in tokenizer, which is what the name result depends on, since the tokenizer decides
how a signature is split. It is trained with the LoRA recipe of the 8B model unchanged
(Appendix~\ref{app:training}), with its own SFT reference per corpus, on the reference,
DPO and the two \method{} variants at three seeds, and without baselines. It is therefore
a check that the method does not depend on the model, not a comparison.
Table~\ref{tab:mistral} reports both corpora.

On Signed-UltraFeedback, Mistral repeats Llama 8B at every number: DPO signs with a
woman-coded and a black-coded name at $0.99$, the two \method{} variants return to $0.50$
and $0.54$ to $0.55$, removing $93$ to $94\%$ and $90$ to $92\%$ of DPO's shift, at two
thirds of DPO's KL and a judge score within DPO's spread. The per-annotator parameters
predict the held-out votes at the planted ceiling on every class, with the pooled scalar
at chance in the class that disagrees (Table~\ref{tab:votes}). On MultiPref, under
DPO the policy produces answers $82$ tokens longer than the reference and opens a gap of
$0.11$, as at 8B. \method{} removes about half of that increase, $48$ to $52\%$ against
$46$ to $50\%$ at 8B, and closes the gap to $0.05$ and $0.04$, with learned scalars of
$0.61$ and $0.66$ against $0.69$ to $0.71$ at 8B. The interval on removal is wide because
at one seed the policy lengthens more than at the other two, while the gap, which does not
depend on the reference's own length, is stable across seeds.

\begin{table}[h]
\centering
\caption{\texttt{Mistral-7B-Instruct-v0.3} with LoRA. Top: Signed-UltraFeedback, columns as in
Table~\ref{tab:names_v3}. Bottom: MultiPref with length declared, columns as in
Table~\ref{tab:multipref}. Mean and 95\% interval over seeds 42, 123 and 456.}
\label{tab:mistral}
\small
\setlength{\tabcolsep}{3pt}
\begin{tabular}{@{}lcccccc@{}}
\toprule
\textbf{Method} & $p(\text{woman})$ & \textbf{Removed} & $p(\text{black})$ & \textbf{Removed} & \textbf{RM$^-$} & \textbf{KL} \\
\midrule
Reference (SFT) & 0.466 & --- & 0.503 & --- & $0.13$ & --- \\
DPO & $0.988 \pm 0.008$ & --- & $0.995 \pm 0.004$ & --- & $1.12 \pm 0.17$ & $0.031 \pm 0.004$ \\
\method{} (pooled) & $0.505 \pm 0.009$ & $93 \pm 2\%$ & $0.552 \pm 0.002$ & $90 \pm 0\%$ & $1.06 \pm 0.16$ & $0.022 \pm 0.003$ \\
\method{} ($\bar\theta+\varepsilon_k$) & $0.499 \pm 0.010$ & $94 \pm 2\%$ & $0.542 \pm 0.004$ & $92 \pm 1\%$ & $1.19 \pm 0.26$ & $0.023 \pm 0.002$ \\
\bottomrule
\end{tabular}

\medskip
\begin{tabular}{@{}lccccc@{}}
\toprule
\textbf{Method} & \textbf{Tokens} & \textbf{Removed} & \textbf{Gap} & \textbf{RM} & \textbf{KL} \\
\midrule
Reference (SFT) & 312.7 & --- & --- & $3.52$ & --- \\
DPO & $394.8 \pm 6.5$ & --- & $0.107 \pm 0.029$ & $5.06 \pm 0.19$ & $0.030 \pm 0.003$ \\
\method{} (pooled) & $355.3 \pm 22.5$ & $48 \pm 29\%$ & $0.049 \pm 0.012$ & $4.70 \pm 0.11$ & $0.028 \pm 0.002$ \\
\method{} ($\bar\theta+\varepsilon_k$) & $351.9 \pm 14.2$ & $52 \pm 19\%$ & $0.041 \pm 0.041$ & $4.79 \pm 0.48$ & $0.029 \pm 0.002$ \\
\bottomrule
\end{tabular}
\end{table}

\section{Vote prediction with the learned bias parameters}
\label{app:votes}

Section~\ref{sec:votes} states that the per-annotator parameters describe the annotators.
This is the measurement behind that claim. On the held-out judgments whose two responses differ on the attribute, each arm predicts the
voter's label from its implicit reward margin plus that voter's learned bias, as defined in
Appendix~\ref{app:readouts}; the ceiling uses the planted $\theta_k$. The ceiling lies below $1$ because the votes are drawn from the likelihood rather than
being deterministic.

Table~\ref{tab:votes} reads the race-coded attribute by the voter's class, since that is
the attribute the classes disagree on, and both attributes overall. Two things stand out.
The $\bar\theta+\varepsilon_k$ variant sits at the ceiling on every class and every model
without being given the class labels: sixty annotators with about four hundred judgments
each are enough to locate every one of them. The pooled scalar cannot do this. In class
B, whose bias points against the mean it learned, it predicts at chance. The
$\bar\theta+\varepsilon_k$ parameters also recover the planted values, correlating with
them at $r = 0.95$ on the gender-coded attribute and $0.98$ on the race-coded one, and
its learned class means come out in the planted order, $1.90$, $-0.22$ and $0.83$ against
the planted $2.62$, $-0.34$ and $1.11$.

\begin{table}[t]
\centering
\caption{Signed-UltraFeedback: share of the held-out judgments whose two responses differ on the attribute that are predicted correctly, by the voter's class for the race-coded attribute and overall for both. The ceiling is the same prediction with the planted $\theta_k$, averaged over the arms' seeds. Means over three seeds. Classes $A$, $B$ and $C$ as in Appendix~\ref{app:names}; class $B$ leans against the race-coded attribute.}
\label{tab:votes}
\small
\setlength{\tabcolsep}{5pt}
\begin{tabular}{@{}lccccc@{}}
\toprule
\textbf{Method} & \textbf{Race, A} & \textbf{Race, B} & \textbf{Race, C} & \textbf{Race, all} & \textbf{Gender, all} \\
\midrule
\multicolumn{6}{@{}l}{\emph{Qwen2.5-0.5B-Instruct, full fine-tuning}} \\
Ceiling (planted $\theta_k$) & 0.860 & 0.694 & 0.703 & 0.752 & 0.749 \\
DPO & 0.774 & 0.497 & 0.650 & 0.641 & 0.673 \\
\method{} (pooled) & 0.765 & 0.500 & 0.678 & 0.649 & 0.694 \\
\method{} ($\bar\theta+\varepsilon_k$) & 0.857 & 0.690 & 0.707 & 0.752 & 0.745 \\
Group-DRO DPO & 0.715 & 0.525 & 0.632 & 0.625 & 0.666 \\
EM-DPO + MinMax-DPO & 0.788 & 0.544 & 0.660 & 0.665 & 0.677 \\
\midrule
\multicolumn{6}{@{}l}{\emph{Llama-3.1-8B-Instruct, LoRA}} \\
Ceiling (planted $\theta_k$) & 0.863 & 0.700 & 0.708 & 0.757 & 0.756 \\
DPO & 0.779 & 0.503 & 0.681 & 0.655 & 0.700 \\
\method{} (pooled) & 0.785 & 0.488 & 0.687 & 0.655 & 0.703 \\
\method{} ($\bar\theta+\varepsilon_k$) & 0.863 & 0.699 & 0.710 & 0.758 & 0.753 \\
Group-DRO DPO & 0.734 & 0.505 & 0.652 & 0.631 & 0.682 \\
\midrule
\multicolumn{6}{@{}l}{\emph{Mistral-7B-Instruct-v0.3, LoRA}} \\
Ceiling (planted $\theta_k$) & 0.853 & 0.712 & 0.717 & 0.761 & 0.753 \\
DPO & 0.745 & 0.530 & 0.676 & 0.652 & 0.679 \\
\method{} (pooled) & 0.749 & 0.513 & 0.676 & 0.647 & 0.685 \\
\method{} ($\bar\theta+\varepsilon_k$) & 0.835 & 0.701 & 0.712 & 0.750 & 0.744 \\
\bottomrule
\end{tabular}
\end{table}

\end{document}